\documentclass[12pt,a4paper,reqno]{amsart}
\usepackage{fullpage}
\usepackage{mathpazo}
\usepackage[english]{babel}

\usepackage[utf8]{inputenc} 
\usepackage[T1]{fontenc}    
\usepackage{hyperref}       
\usepackage{url}            
\usepackage{booktabs}       
\usepackage{amsfonts}       
\usepackage{nicefrac}       
\usepackage{microtype} 

\usepackage{csquotes}

\usepackage{amsmath,amssymb,amsthm,bbm,mathtools}

\usepackage[normalem]{ulem}

\usepackage{xy}
\input xy \xyoption{all}

\usepackage{makecell}

\usepackage{cleveref}
\Crefname{dfn}{Definition}{Definitions}

\usepackage[dvipsnames]{xcolor}
\usepackage{tcolorbox}

\usepackage{tikz}    
\usepackage{pgfplots}
\usepackage{mathrsfs}  

\usepackage{enumerate}\usepackage[inline]{enumitem}

\newcommand{\ra}[1]{\renewcommand{\arraystretch}{#1}}
\usepackage{caption}

\usepackage[
 backend=biber,
 style=numeric,
 giveninits=true,
 maxcitenames = 10,
 maxbibnames = 10
 ]{biblatex}
 \DeclareSourcemap{
  \maps[datatype=bibtex]{
    \map{
      \step[fieldset=isbn, null]
      \step[fieldset=issn, null]
      \step[fieldset=doi, null]
      \step[fieldset=url, null]
      \step[fieldset=mrclass, null]
      \step[fieldset=mrnumber, null]
      \step[fieldset=note, null]
      \step[fieldset=abstract, null]
    }
    \map[overwrite=true]{
      \step[fieldsource=fjournal]
      \step[fieldset=journal, origfieldval]
    }
  }
}
\definecolor{applegreen}{rgb}{0.55, 0.71, 0.0}
\definecolor{lightapplegreen}{rgb}{0.95, 1, 0.9}
\definecolor{alizarin}{rgb}{0.82, 0.1, 0.26}
\definecolor{lightalizarin}{rgb}{1, 0.95, 0.98}
\definecolor{ablue}{rgb}{0.36, 0.54, 0.66}
\definecolor{lightablue}{rgb}{0.92, 0.95, 0.97}
\usepackage{empheq}

\newcommand{\eps}{\epsilon}
\newcommand{\be}{\beta}
\newcommand{\Th}{\Theta}
\let\th\relax\newcommand{\th}{\theta}
\newcommand{\la}{\lambda}

\newcommand{\si}{\sigma}

\newcommand{\Ga}{\Gamma}

\newcommand{\de}{\delta}

\newcommand{\N}{\mathbb{N}}

\newcommand{\R}{\mathbb{R}}

\let\aa\relax \newcommand{\aa}{\mathcal{A}}
\newcommand{\cc}{\mathcal{C}}
\newcommand{\rr}{\mathcal{R}}
\newcommand{\xx}{\mathcal{X}}
\newcommand{\yy}{\mathcal{Y}}
\newcommand{\zz}{\mathcal{Z}}

\newcommand{\mm}{\mathcal{M}}
\newcommand{\nn}{\mathcal{N}}
\newcommand{\hh}{\mathcal{H}}
\newcommand{\bb}{\mathcal{B}}
\newcommand{\ff}{\mathcal{F}}

\newcommand{\cS}{\mathcal{S}}

\newcommand{\ii}{\mathcal{I}}

\newcommand{\D}{{\rm d}}

\newcommand{\scal}[4]{{}_{#3}\langle #1,#2\rangle_{#4}}
\newcommand{\nor}[2]{\| #1\|_{#2}}

\DeclareMathOperator{\ran}{Ran}

\DeclareMathOperator{\spn}{span}

\DeclareMathOperator{\TV}{TV}

\DeclareMathOperator{\Ext}{Ext}

\DeclareMathOperator*{\argmin}{arg\,min}

\theoremstyle{plain}
\newtheorem{thm}{Theorem}[section]
\newtheorem{lem}[thm]{Lemma}
\newtheorem{prop}[thm]{Proposition}
\newtheorem{cor}[thm]{Corollary}

\theoremstyle{plain} 
\newtheorem{dfn}[thm]{Definition}
\newtheorem{asp}{Assumption}
\theoremstyle{plain} 

\crefname{exa}{Example}{Examples}

\theoremstyle{plain} 
\newtheorem{rmk}[thm]{Remark} 

\crefname{rmk}{remark}{remarks}

\newcommand{\ern}[1]{{\color{black}#1}}

\begin{document}
\title{\bf Neural reproducing kernel Banach spaces \\
and representer theorems for deep networks
}

\author{F.~Bartolucci}
\address[F.~Bartolucci]{Analysis Group - Delft Institute of Applied Mathematics, TU Delft, Netherlands}
\email{f.bartolucci@tudelft.nl}

\author{E.~De Vito}
\address[E. De Vito]{MaLGa - DIMA, University of Genoa,
Italy
}
\email{ernesto.devito@unige.it}

\author{L.~Rosasco}
\address[L. Rosasco]{MaLGa - DIBRIS, University of Genoa,
Italy
\& CBMM, MIT \& IIT}
\email{lorenzo.rosasco@unige.it}

\author{S.~Vigogna}
\address[S. Vigogna]{RoMaDS - Department of Mathematics, University of Rome Tor Vergata,
Italy
}
\email{vigogna@mat.uniroma2.it}

\maketitle

\begin{abstract}
Characterizing the function spaces defined by  neural networks helps understanding the corresponding learning models and their inductive bias. 
While in some limits neural networks  correspond to function spaces that are Hilbert spaces,  these regimes do not capture  the properties of the networks used in practice.  
Indeed, several results have shown that shallow networks can be better characterized in terms of suitable Banach spaces. However, analogous results for deep networks are limited.
In this paper we show that deep neural networks  define suitable reproducing kernel Banach spaces.
 These spaces are equipped with  norms that enforce a form of sparsity, enabling them to adapt to potential latent structures within the input data and their representations. In particular, by leveraging the theory of reproducing kernel Banach spaces, combined with variational results, we derive representer theorems that justify the finite architectures commonly employed in applications. Our study extends analogous results for shallow networks and represents a step towards understanding the function spaces induced by neural architectures used in practice.
\end{abstract}

\section{Introduction}

Neural networks define functions by composing linear and nonlinear maps in a multi-layer (deep) architecture. While easy to implement,  the corresponding models are hard to analyze since they are 
nonlinearly parameterized. The study of the function spaces defined by different neural network architectures provides insights into the corresponding learning models. 
In particular, it provides indications on the underlying  inductive bias, namely,  which functions can be approximated and learned efficiently by a given class of networks. 

In some over-parameterized regimes, neural networks can be seen to define Hilbert spaces of functions and in particular 
reproducing kernel Hilbert spaces (RKHS) \cite{6fc4e20b-e6b8-3651-8988-bc9975065f31}. It is a classical observation that shallow networks  with infinitely many random units correspond to RKHS,
with reproducing 
kernels depending on the considered activation function \cite{neal2012bayesian}.
This regime, also known as  the Gaussian Process (GP) limit, has connections with models such as random features \cite{NIPS2007_013a006f}. The limits of more complex, possibly non-shallow  architectures can also be derived and characterized in terms of RKHS, see e.g.  \cite{hanin}.
Another infinite-width limit in which neural networks are described by RKHS is the so-called lazy training regime \cite{chizat-lazy}.
In this limit,  the network weights evolve little during the optimization and can be well approximated by a linear approximation around a random initialization.  Also in this case,  the corresponding function spaces are RKHS, and the associated kernel is called neural tangent kernel (NTK) \cite{JacotHG18}. Again, neural tangent kernels and corresponding RKHS can be derived for a variety of architectures, see e.g. \cite{NEURIPS2019_c4ef9c39}.

However, neither the above GP/kernel limit nor the NTK/lazy training regime appear  
to capture key aspects of neural network models \cite{NEURIPS2020_a9df2255,bietti21} used in practice.
Results for shallow networks 
suggest that neural networks might favor functions with small norms that are not 
Hilbertian but rather associated with Banach spaces \cite{JMLR:v18:14-546}. In turn, representer theorems 
associated to such norms allow to derive finite-width networks commonly used in practice from variational principles \cite{savarese2019infinite,ongie2019function,unser2020unifying,parhi2021banach,doi:10.1137/21M1418642}.
These observations have sparked   interest in understanding  Banach spaces 
associated to neural networks. One possibility is to consider extensions of classical splines 
\cite{unser2017splines,unser2019native,JMLR:v24:22-0227,unser2023kernel}. Another possibility is to consider reproducing kernel Banach spaces (RKBS) \cite{JMLR:v10:zhang09b,Lin2019OnRK}, see e.g. \cite{BARTOLUCCI2023194} and references therein.

In this paper, we develop the latter perspective tackling the extension from shallow to deep networks.  
The study of Banach spaces associated to deep architectures and corresponding representer theorems was started in \cite{doi:10.1137/21M1418642}, where deep architectures  with ReLU activations and  finite rank constraints at each layer are considered. The latter requirement is not natural and  is mainly due  to technical reasons. Indeed,  finite rank constraints allow for the construction of  layers as concatenation of vector valued functions studied for shallow networks. In our study, we propose an approach that avoids finite rank constraints and allows to consider more general activations. This requires more substantial developments employing vector  measures of finite variation to address the challenges posed by potentially  infinite-dimensional hidden layers. 
 Our first contribution is to define a reproducing kernel Banach space which describes an infinite-width limit of a 
 deep neural network with an associated norm promoting sparsity. We call such a space a neural RKBS. 
 Then,  we provide a representer theorem for a large class of nonlinearities that shows how   neural networks minimizing empirical objectives can be taken to have a finite width at every layer. This result  extends analogous results for shallow networks. It implies that commonly used networks are optimal in the sense that they are  solutions of a suitable variational problem. 
 
The rest of the paper is organized as follows. In \Cref{sec:preliminaries} we provide background for the study of the paper. In \Cref{sec:vector-valued-RKBSs} we review the main technical ingredients of our construction,
namely reproducing kernel Banach spaces and vector Radon measures. In \Cref{sec:neuralRKBS} we introduce deep integral RKBS to model functional properties of deep neural networks, leading to the construction of neural RKBS. In \Cref{sec:representer-theorems} we prove our represent theorems for deep neural networks. In \Cref{sec:discrete-neural-rkbs}, we construct a particular instance of neural
RKBS with a countably infinite number of neurons per hidden layer, which permits a more explicit form of the relative representer theorem.
In \Cref{sec:appendix} we collect variational results and extreme point characterizations used to prove our representer theorems. Table~\ref{tab:notation} summarizes  the main notation we use in the
paper.
\begin{table}[h]

\caption{Notation} \label{tab:notation}

\centering

\ra{1.5}

\resizebox{\columnwidth}{!}{

\begin{tabular}{l l l l}

\textsf{symbol} & \textsf{definition} & \textsf{symbol} & \textsf{definition} \\

\hline
$\mathcal{X},\yy$ & Banach spaces & $\Theta$ & locally compact second contable
                                   space \\

$ B(\mathcal{X},\yy) $ & bounded linear maps $ \mathcal{X} \to \yy $ 
& $\bb(\Theta) $ & Borel
$\sigma$-algebra on $\Theta$ \\

$ \mathcal{X}' $ & continuous dual of $\mathcal{X}$
& $ \cc_0(\Theta,\yy) $ & continuous functions $ \Theta \to \yy $ vanishing at $\infty$
\\

$ {}_\mathcal{X}\langle \cdot , \cdot \rangle_{\mathcal{X}'} $ & pairing on $\mathcal{X},\mathcal{X}'$
&
$ \cc_0(\Theta) $ & $ \cc_0(\Theta,\R) $ \\

$ \langle \cdot , \cdot \rangle_\mathcal{X} $ & inner product on the Hilbert
                                      space $\mathcal{X}$
 & $ \mm(\Theta,\yy) $ & vector measures on $\Theta$ with values in $\yy$

\\

$ \| \cdot \|_\mathcal{X} $ & norm on $\mathcal{X}$
& $ \mm(\Theta) $ & $ \mm(\Theta,\R) $
\\

$ B_\mathcal{X}(r) $ & ball on $\mathcal{X}$ of radius $r$  & $ \delta_\theta $  & Dirac delta at $\theta$ \\

$\Ext(Q)$ & extremal points of $Q\subset \mathcal{X}$ & $ \| \cdot \|_{\TV} $ & total variation norm
 \\

\hline

\end{tabular}

}

\end{table}

\section{Preliminary Discussion} \label{sec:preliminaries}

In this section, we provide the necessary background by introducing the notation for deep neural networks, reviewing the infinite-width limit of shallow networks, and presenting a high-level overview of the proposed infinite-width limit for deep architectures.

\subsection*{Neural networks}

We start by setting up some notation and
introduce  fully connected feed-forward neural networks.

\begin{dfn}[Fully connected feed-forward neural network] \label{def:nn}
Let $ \sigma : \R \to \R $ be a (nonlinear) function,
$ L \ge 1 $ an integer
and
\[ d=d_0, d_1,\dots,d_L,d_{L+1}=p \ge 1 \] 
a family of $L+2$  integers.
A function $ f : \R^d \to \R^p $  is called a 
fully connected feed-forward neural network from
$\R^d$ to $\R^p$ with activation function $\sigma$, depth $L$ and widths $ d_1,\dots,d_L $ if
\[
f(x)  = x^{(L+1)} 
\]
where, for each $x\in\R^d$, the vector $x^{(L+1)}\in\R^p$ is defined by  the following recursive equation
\begin{equation} \label{eq:nn}
 \begin{cases}
 x^{(1)} = W^{(1)} x + b^{(1)} & \in \R^{d_1} \\[5pt]
  x^{(\ell+1)} = W^{(\ell+1)} \sigma( x^{(\ell)} ) + b^{(\ell+1)} &
  \in \R^{d_{\ell+1}}, \qquad \ell = 1 , \dots , L,
 \end{cases}
\end{equation}
for some weights $ W^{(\ell)} \in \R^{d_\ell
  \times d_{\ell-1}} $ and biases (or offsets)
$ b^{(\ell)} \in \R^{d_\ell} $.
We call a neural network shallow (or a one-hidden layer network) if $L=1$, deep if $L>1$.
\end{dfn}
In~\eqref{eq:nn} the activation function $\sigma$ is assumed to
be applied on vectors component-wise, and the vector $ \sigma( x^{(\ell)} ) $ for $ \ell = 1 , \dots , L $
is the $\ell$-th \emph{hidden layer} of the network.
The input and output dimensions $d=d_0,p=d_{L+1}$
are fixed by the problem, and $ d_\ell $ is the number of neurons (or
units)  at the $\ell$-th hidden layer. A \emph{neuron} is a function of the form $ \phi(z) = \sigma( \langle z , w \rangle + b ) $.
Hence, the definition of a network requires specifying an activation function $\sigma$, a depth $L$, and the widths $d_1,\dots,d_L$. These parameters define the network architecture. Given an architecture, the parameters $W^{(\ell)},b^{(\ell)}$ in $ \R^{d_\ell \times d_{\ell-1}} \times \R^{d_\ell} $ are found optimizing an empirical objective function.

The neural networks with a fixed architecture form a subset  $\mathcal F_{NN}$ of
functions from $\R^d$ to $\R^p$ parameterized nonlinearly by weights and biases.
This is already clear considering scalar-valued shallow neural networks of width $K\in\N$, that is
\begin{equation}\label{eq:onehidddennns}
\mathcal F_{NN}=\left\{f:\R^d \to \R : f(x)=\sum_{k=1}^K v_k\sigma(w_k^\top x +b_k)+b, w_k\in\R^d, v_k, b, b_k\in\R\right\}.
\end{equation}

The nonlinear dependence on the parameters is a key challenge in defining
and characterizing   the corresponding function spaces. For shallow networks,  considering
the so called infinite-width limit provides an approach to tackle this
question.

\subsection*{Infinite-width limit of shallow neural networks.} 
In the context of shallow networks, the  infinite-width limit corresponds to consider 
functions  parameterized by measures rather than weights, that is
\begin{equation}
\label{eq:infinite_width_nn}
    f_{\mu}(x) = \int_{\R^d\times\R} \sigma( w^\top x + b)\, {\rm d}\mu(w,b),\qquad \mu \in \mathcal{M}(\R^d\times\R).
\end{equation}
Generalizing further, one can consider an arbitrary input space $\xx$
and functions $f$ from $\xx$ to $\R$ of the form
\begin{equation}
\label{eq:infinite_width_nn-rho}
    f_{\mu}(x) = \int_\Theta \rho(x,\th)\, {\rm d}\mu(\th),\qquad \mu \in \mathcal{M}(\Theta)  .
\end{equation}
for a suitable locally compact second countable parameter space $\Th$
and basis function $ \rho : \xx \times \Th \to \R $.

We add two observations.
First,  the  expression in \eqref{eq:infinite_width_nn} needs some care. Indeed,  commonly 
used activation functions are not integrable, and typical parameter spaces $\Theta$ are non-compact,
so that the integrals in \eqref{eq:infinite_width_nn} (and \eqref{eq:infinite_width_nn-rho}) may not converge.
For example,  in \cite{BARTOLUCCI2023194}  the activation is multiplied by a smoothing function to ensure integrability. The corresponding function space is shown to be a  reproducing kernel Banach space, with norm  induced by the total variation norm on measures. 
Second, the finite shallow networks in the set \eqref{eq:onehidddennns}
can be recovered from the infinite-width model \eqref{eq:infinite_width_nn} by taking finite linear combinations of Dirac deltas as measures. Interestingly, finite networks can be shown to emerge from variational principles. Specifically, if the minimization of an empirical objective is performed over measures,  instead of weights, then the optimal solutions are atomic measures. In other words, finite networks are optimal solutions of empirical minimization problems  over possibly infinite dimensional networks.  We refer for example to \cite{unser2023kernel}  for an account of recent works  addressing the above topics. As discussed next, our main motivation is studying similar questions in the context of deep networks.

\subsection*{Infinite-width limit of deep neural networks}
Previous work on Banach spaces associated with deep architectures, and corresponding representation theorems, focus on layers which are concatenation of  of infinite-width shallow neural networks with finite-dimensional outputs. As a result, this approach inherently restricts the deep architecture to finite-rank constraints at every layer \cite{doi:10.1137/21M1418642}. In contrast, our work presents an alternative framework that avoids such constraints and accommodates a broader class of activation functions. This generalization requires addressing the challenges posed by potentially infinite-dimensional hidden layers. To this end, we construct deep architectures as compositions of RKBS of integrable functions valued in infinite-dimensional spaces. More precisely, we begin by considering the direct sum of integral RKBS, forming a linear space, and use this structure to parameterize spaces of composed functions, thereby yielding a corresponding nonlinear function space, which we call a deep integral RKBS. By further specializing this construction to integral functions defined through activation functions, we arrive at what we call a neural RKBS. This formulation provides a natural and rigorous connection with commonly used deep neural network models. Finally, we use these function spaces to derive novel representer theorems for deep networks. These representer theorems characterize the minimization of empirical objective functions over deep RKBS. In particular, they show that finite networks are optimal from a variational perspective. Next, we develop these ideas in detail.

\section{Vector measures, integral RKBS and  Neural Networks}
\label{sec:vector-valued-RKBSs}

We introduce the following definition,
that readily generalizes vector-valued reproducing kernel Hilbert spaces \cite{MR2265340} to a Banach setting.
We refer to \cite{JMLR:v10:zhang09b,Lin2019OnRK} for an overview.

\begin{dfn}[Vector-valued RKBS]
 Let $\xx$ be a set and $\yy$ a Banach space.
 A \emph{reproducing kernel Banach space} (RKBS) $\hh$ on $\xx$ with
 values in $\yy$ is a Banach space such that
 \begin{enumerate}[label=\textnormal{(\roman*)}, itemsep=5pt]
 \item the elements of $\hh$ are  functions $ f : \xx \to \yy $;
\item the sum and the multiplication by scalars in $\hh$ are defined pointwise;
\item for all $ x \in \xx $ there is $ C_x > 0 $ such that $ \|f(x)\|_\yy \le C_x \|f\|_\hh $ for all $ f \in \hh $.
 \end{enumerate}
\end{dfn}
The first two conditions are equivalent to say that $\hh$ is a subspace of
$\yy^{\xx}$, the vector space of all functions from $\xx$ to $\yy$,  while
the third  condition requires that  for all $ x \in \xx $ the pointwise evaluation
$ f \mapsto f(x) $ is in $ B(\hh,\yy) $, the space of bounded linear operators between $\mathcal{H}$ and $\yy$.

Reproducing kernel Hilbert spaces can be characterized in terms of so-called feature maps.
In the following proposition we  provide an analogous result for
RKBS. In this case, feature spaces are  Banach spaces. 

\begin{prop} \label{prop:featureRKBS}
Let $\xx$ be a set, $\yy$ a Banach space and $\hh$ a set of functions
$ f : \xx \to \yy $.  Consider the following statements.
\begin{enumerate}[label=\textnormal{(\alph*)}, itemsep=5pt]
\item \label{it:rkbs}
The space $\hh$ is a RKBS.
\item \label{it:phi}
There is a Banach space $\ff$
 and a map $ \phi : \xx \to B(\ff,\yy) $ such that
 \vspace{5pt}
 \begin{enumerate}[label=\textnormal{(\roman*)}, itemsep=5pt]
  \item \label{it:B}
  $ \hh = \{ f_\mu : \mu \in \ff \} $
  where $ f_\mu = \phi(\cdot) \mu $ ;
  \item \label{it:normB}
  $ \| f \|_\hh = \inf \{ \| \mu \|_\ff : \mu \in \ff , f = f_\mu \} $ .
 \end{enumerate}
\item \label{it:psi}
There is a Banach space $\ff$
 and a map $ \psi : \xx \to B(\yy',\ff') $ such that
 \vspace{5pt}
 \begin{enumerate}[label=\textnormal{(\roman*)}, itemsep=5pt]
  \item \label{it:B'}
  $ \hh = \{ f_\mu : \mu \in \ff \} $
  where $ {}_{\yy''}\langle f_\mu(\cdot) , y' \rangle_{\yy'} = {}_\ff\langle \mu , \psi(\cdot) y' \rangle_{\ff'} $ for all $y'\in\yy'$;
  \item
  $ \| f \|_\hh = \inf \{ \| \mu \|_\ff : \mu \in \ff , f = f_\mu \} $ .
 \end{enumerate}
\end{enumerate}
Then \ref{it:rkbs} and \ref{it:phi} are equivalent and each one implies \ref{it:psi}.
Moreover, if $\yy$ is reflexive (in particular Hilbert), then \ref{it:rkbs}, \ref{it:phi} and \ref{it:psi} are all equivalent.
\end{prop}
\begin{proof}
 To see that \ref{it:rkbs} implies \ref{it:phi},
 take $ \ff = \hh $ and define
 \[
\phi : \xx \to B(\ff,\yy), \qquad \phi(x) f = f(x) .
   \]
 Then \ref{it:B} and \ref{it:normB} of item \ref{it:phi} are clear.
 Let us prove that \ref{it:phi} implies \ref{it:rkbs}.
 Clearly $\hh$ is a linear space and $\|\cdot\|_\hh$ is a norm.
 We then show that the normed space $\hh$ is complete.
The linear map $ \mu \mapsto f_\mu $ has kernel
$ \nn = \bigcap_{x\in\xx} \ker \phi(x) $.
Since $ \phi(x) $ is bounded for all $x\in\xx$, $ \ker \phi(x) $ is closed, hence so is $\nn$.
Thus, $ \ff / \nn $ is a Banach space \cite[Theorem 1.41]{rudin91} isomorphic to $\hh$ by construction,
which is therefore complete.
Next, we show that point evaluations in $\hh$ are continuous.
For $ f \in \hh $, let $ \mu \in \ff $ such that $ f = f_\mu $.
Then
$$
 \| f(x) \|_\yy = \| \phi(x)\mu \|_\yy \le \| \phi(x)\|_{B(\ff,\yy)} \|\mu\|_\ff  ,
$$
whence
\begin{equation*}
 \| f(x) \|_\yy \le \inf_{\mu\in\ff : f = f_\mu} \| \phi(x) \|_{B(\ff,\yy)} \| \mu \|_\ff
 = \| \phi(x) \|_{B(\ff,\yy)} \| f \|_\hh  .
\end{equation*}
The implication from \ref{it:phi} to \ref{it:psi} follows easily considering
 $ \psi(x)=\phi(x)^t\in\mathcal B(\yy',\ff') $ be the transpose map of $\phi(x)$.
Finally,
note that \ref{it:B'} in \ref{it:psi} defines
$f_\mu$ as a function from $\xx$ to $\yy''$.
Hence, if $\yy$ is reflexive,
it defines $ f_\mu : \xx \to \yy $.
From here,
following the proof of the implication
from \ref{it:phi} to \ref{it:rkbs},
one can prove that \ref{it:psi} implies \ref{it:rkbs}.
\end{proof}

\subsection*{Vector Radon measures}

In view of \Cref{prop:featureRKBS}, a 
RKBS can be constructed choosing a suitable feature space.
Thinking of a neural network layer as an atomic integration, we will define RKBS parameterized by measures. Since layers have vectorial outputs, we need the notion of vector valued measure \cite{diestel1977vector}.

Let $\Theta$ be a Hausdorff, locally compact, second countable topological space, and let $\yy$ be a Banach space. Recall that a (numerable) partition of a Borel set $A$ is a numerable family of Borel sets $\{A_i\}$ such that $ A_i \cap A_j=\emptyset $ for all $ i \ne j $ and $ \bigcup_i A_i = A $. 
\begin{dfn}[Vector measure]
A \emph{vector  measure} on $\Theta$ with values in $\yy$ is a set function $\mu\colon\mathcal{B}(\Theta)\to\yy$ such that 
for all $A\in\mathcal{B}(\Theta)$ and all $\{A_i\}$ partitions of $A$,
$$
 \mu(A)=\sum_i\mu(A_i)  ,
$$
where the sum converges unconditionally in the $ \| \cdot \|_\yy $-norm.
\end{dfn}
Pettis theorem \cite[Thm.1 IV.10.1]{dunford1963linear} shows that a set function $\mu\colon\mathcal{B}(\Theta)\to\yy$ is a vector measure if and only if $\mu_{y'}=\scal{\mu(\cdot)}{y'}{\yy}{\yy'}$ is a $\sigma$-additive measure for all $y'\in\yy'$.
 \begin{dfn}[Variation of a vector measure]
 Let $\mu$ be a vector measure on $\Theta$ with values in $\yy$. 
 The variation of $\mu$ is the function $|\mu|\colon\mathcal{B}(\Theta)\to [0,+\infty]$ defined by 
 \[
 |\mu|(A)=\sup_{\{A_i\}}\sum_i\|\mu(A_i)\|_{\yy} \qquad A\in\mathcal{B}(\Theta) ,
 \]
 where the supremum is taken over all finite partitions of $A$.
If $|\mu|(\Theta)<+\infty$, the measure $\mu$ is called a vector measure of bounded variation.
\end{dfn}
The space $\mathcal{M}(\Theta,\yy)$ of vector measures of bounded variation is a Banach space with respect to the norm 
\[
\|\mu\|_{\rm TV}=|\mu|(\Theta)  .
\]
If $\mu\in\mathcal{M}(\Theta,\yy)$, its variation $|\mu|$ is a finite positive measure on $\Theta$, see \cite[1.A.10]{dinculeanu2000vector}.

The integration of a scalar function $\varphi$ with respect a  vector measure of bounded variation can be defined as the Bochner integral of a vector valued function, see \cite[ Ch.1 Section D]{dinculeanu2000vector} and \cite[IV.10]{dunford1963linear}. In particular, a measurable scalar function $\varphi$ is integrable with respect to $\mu$ if and only if $\varphi$ is integrable with respect to $|\mu|$. 
If $\varphi=\sum_i t_i {\chi}_{A_i}$ is a simple function, then 
\[
\int_{\Theta} \varphi(\theta) d\mu(\theta)= \sum_i t_i \mu(A_i)\in \yy,
\]
and the integral of an arbitrary $|\mu|$-integrable functions is defined via the density of simple functions.

If $\yy$ has the Radon-Nikodym property \cite[Ch. 1.G]{dinculeanu2000vector}, 
as it happens if $\yy$ is reflexive, then $\mu$  has density  with respect to $|\mu|$, {\it i.e.} there exists a  function $g\colon\Theta\to\yy$ such that $\nor{g(\theta)}{\yy}=1$ for all $\theta\in\Theta$, $g$ is Bocnher integrable with respect to $|\mu|$, and
\[
\int_{\Theta} \varphi(\theta) d\mu(\theta)= \int_{\Theta} \varphi(\theta) g(\theta) d|\mu|(\theta),
\]
for all  $\mu$-integrable scalar functions $\varphi:\Theta\to\R$. If $\yy$  does not have the Radon-Nikodym property, 
nevertheless there always exists a function $g:\Theta\to \yy$ such that
\begin{enumerate}[label=\roman*)]
    \item for all $y'\in\yy'$, $\scal{g(\cdot)}{y'}{\yy}{\yy'}$ is $|\mu|$-integrable; 
    \item for all $\theta\in \Theta$, $\nor{g(\theta)}{\yy}=1$;
    \item for all scalar function $\varphi$ that are $\mu$-integrable and $y'\in\yy'$
    \[
\scal{\int_{\Theta} \varphi(\theta) d\mu(\theta)}{y'}{\yy}{\yy'}= \int \varphi(\theta) \scal{g(\theta)  }{y'}{\yy}{\yy'}d|\mu|(\theta),
    \]
\end{enumerate}
see \cite[ Ch.1 Theorem 34]{dinculeanu2000vector}.

\subsection*{Vector integral RKBS and neural networks}
We now introduce particular classes of RKBS.
In particular,  we extend the infinite-width limit \eqref{eq:infinite_width_nn-rho} of shallow neural networks from scalar to vector-valued functions.
This result is of independent interest and will be crucial for the extension from shallow to deep networks. 

\begin{dfn}[Integral RKBS] \label{def:int_rkbs}
Let $ \Theta $ be a locally compact, second countable topological space,
$\xx$ a set, and $\yy$ a Banach space.
Let
$$
\rho : \xx \times \Theta \to \R
$$
be such that
$ \rho(x,\cdot) \in \cc_0(\Theta) $ for all $ x \in \xx $.
For $ \mu \in \mm(\Theta , \yy) $, let
\begin{equation}\label{eq:integral}
    f_\mu : \xx \to \yy, \qquad f_\mu(x)
 = \phi(x)\mu
 = \int_{\Theta} \rho(x,\theta) \D\mu(\theta),
\end{equation}
and
$$
 \hh = \{ f_\mu : \xx \to \yy : \mu \in \mm(\Th , \yy)  \},
$$
with
\begin{equation}\label{eq:norma}
    \| f \|_\hh = \inf_{\mu \in \mm(\Theta , \yy)} \{ \| \mu \|_{\TV} : f = f_\mu \} .
\end{equation}
The space $ \hh $ thus defined is a (vector-valued) RKBS,
which we call an \emph{integral RKBS}.
Moreover, we call $\Theta$ the \emph{parameter spaces} of $\hh$,
and $ \rho $ its \emph{basis function}.
\end{dfn}
\begin{rmk}
 If $ \rho(x,\cdot)$ is bounded for all $ x \in \xx $,~\eqref{eq:integral} is still well defined. The stronger assumption that $\rho(x,\cdot) \in \cc_0(\Theta) $ ensures that the map $f\mapsto f(x)$ is weakly$^*$ continuous, see \Cref{riesz}.
 \end{rmk}

\begin{rmk}\label{rmk:our_shallow}
    Scalar valued integral RKBS correspond to the choice $\yy=\mathbb{R}$
    and coincide with the setting considered in \cite{BARTOLUCCI2023194}.
\end{rmk}

\section{Deep Integral and Neural RKBS} \label{sec:neuralRKBS}

In this section, we start introducing \emph{deep} integral RKBS. Then, we derive representer theorems on these spaces, which show how optimal solutions minimizing empirical objectives can be taken to have a finite width at every layer.

\subsection*{Deep RKBS}
The first step is to take direct sums of RKBS and define the deep
  RKBS as a composition of elements in the direct sum.
\begin{dfn}[Deep RKBS] \label{def:deep_rkbs}
Let $\xx$ be a set and $\yy$ a Banach space.  Fix a positive integer
$L\geq 1$. 
Take a set $\xx_0=\xx$
and Banach spaces $ \xx_1 , \dots , \xx_{L+1}=\yy $.
For $ \ell = 0 , \dots , L $,
take RKBS $\hh_\ell$ on $\xx_\ell$ with values in $ \xx_{\ell+1} $.
The direct sum
$$
 \hh = \hh_0 \oplus \cdots \oplus \hh_L
$$
is a Banach space with respect to the norm
$$
 \| f \|_\hh = \| f_0 \|_{\hh_0} + \cdots + \| f_L \|_{\hh_L}, \qquad f
 = f_0 \oplus \cdots \oplus f_L . $$ 
To every $ f =f_0 \oplus \cdots \oplus f_L\in \hh $
we assign the function $ f^{\rm deep} : \xx \to \yy $ defined by
$$
 f^{\rm deep} = f_L \circ \cdots \circ f_0,
$$
and we set
$$
 \hh^{\rm deep} = \{ f^{\rm deep} : \xx \to \yy : f \in \hh \} \ 
$$
endowed  with the complexity measure $\Phi: \hh^{\rm deep}\to [0,+\infty)$ given by
  \begin{equation}
\Phi(f^{\rm deep})=\inf\{ \| g \|_\hh : g\in\hh\ \text{such that } g^{\rm deep} = f^{\rm deep} \}.\label{eq:4}
\end{equation}

With a slight abuse of language, we call the nonlinear space $
  \hh^{\rm deep}$  an \emph{RKBS of depth $L$} induced by the Banach space
  $\hh$. If $L>1$, we refer to
  $\hh^{\rm deep}$  as a \emph{deep 
    RKBS}. 
Moreover, we call $ \xx_\ell $ the \emph{layer spaces} of $\hh$.
In particular, $ \xx = \xx_0 $ is the \emph{input space},
$ \yy = \xx_{L+1} $ is the \emph{output space},
and $ \xx_\ell $, $ \ell = 1 , \dots, L $, are the \emph{hidden layer spaces}.
\end{dfn}

\subsection*{Deep integral RKBS}
Next, we combine \Cref{def:int_rkbs,def:deep_rkbs}.
\begin{dfn}[Deep integral RKBS]\label{dfn:deep-integral-rkbs}
Let $\hh^{\rm deep}$ be an RKBS of depth $L$.
If the RKBS $\hh_\ell$ are integral for all $\ell=0,\dots,L$,
$$
 \hh_\ell = \{ f_{\mu_\ell} : \xx_\ell \to \xx_{\ell+1} : \mu_\ell \in \mm(\Theta_\ell,\xx_{\ell+1}) \} ,
$$
with basis functions
$$
 \rho_\ell : \xx_\ell \times \Theta_\ell \to \R  ,
$$
where $ \rho_\ell(x,\cdot) \in \cc_0(\Theta_\ell) $ for all $ x \in \xx $. We call $\hh^{\rm deep}$
an \emph{integral RKBS of depth $L$} and, 
if $L>1$, a \emph{deep integral RKBS}.
\end{dfn}

Thus, a function $ f^{\rm deep} : \xx \to \yy $ in a deep integral RKBS $\hh^{\rm deep}$ of depth $L$ has the form
\begin{equation}\label{eq:deep_integral}
 \begin{cases}
  x^{(0)} = x & \in \xx \\[5pt]
  x^{(\ell+1)} = \displaystyle{\int_{\Theta_\ell} \rho_\ell(x^{(\ell)},\theta_\ell) \D\mu_\ell(\theta_\ell)}
  & \in \xx_{\ell+1}, \qquad \ell = 0 , \dots , L, \\
   f^{\rm deep}(x) = x^{(L+1)}  & \in \yy\
 \end{cases}
\end{equation}
where $\mu_\ell \in \mm(\Th_\ell , \xx^{\ell+1})$ for any
$\ell=0,\ldots,L$.

We now define (deep) integral RKBS modeled on neural networks.
In such spaces, the basis functions are defined in terms of an activation function.

\begin{dfn}[Neural RKBS]\label{dfn:neural-rkbs}
 Let $\hh^{\rm deep}$ be an integral RKBS of depth $L$.
 Suppose that, for each $\ell$, the layer $ \xx_\ell  $ is a function space over the parameter space $ \Th_\ell $, that is, $ \xx_\ell \subset \R^{\Th_\ell} $.
Let $ \sigma : \R \to \R $ be a (nonlinear) activation function
and $ c_\ell : \Th_\ell \to \R $.
Suppose that the basis functions $\rho_\ell$ are of the form
 \begin{align*}
  & \rho_0(x_0,\th_0) = x_0(\th_0) \\
  & \rho_\ell ( x_\ell , \th_\ell ) = \sigma ( x_\ell(\th_\ell) + c_\ell(\th_\ell) ) \be_\ell(\th_\ell),
  \qquad \ell = 1, \dots , L ,
 \end{align*}
 where $ \be_\ell : \Th_\ell \to \R $ is such that $ \rho_\ell(x_\ell,\cdot) \in \cc_0(\Th_\ell) $ for all $ x_\ell \in \xx_\ell $.
Then we call $\hh^{\rm deep}$ a \emph{neural RKBS} of depth $L$ and, if $L>1$, a \emph{deep} neural RKBS.
\end{dfn}

\begin{rmk}[Infinite-width shallow neural networks revisited]\label{bartolucci}
    In Remark~\ref{rmk:our_shallow}, we observed that function spaces considered in \cite{BARTOLUCCI2023194} correspond to the choice $\yy=\R$ in Definition~\ref{def:int_rkbs}. Here, we show that those spaces can also be recovered from Definition~\ref{dfn:neural-rkbs} by taking $L=1$, parameter and layer spaces
    \begin{alignat*}{2}
      &  \Theta_0 = \{ 0 , \dots, d \}, \qquad && \xx_0 = \R^d, \\
 & \Theta_1 = \R^{d+1}, \qquad &&
 \xx_1 = \R^{d+1},\\
 & \qquad && \xx_2 = \R \ ,
    \end{alignat*}
and basis functions of the form
\begin{align*}
 & \rho_0 (x , j) =
 \begin{cases}
   1  & \hspace{29pt} j=0 \\
   x_j  & \hspace{29pt} j=1,\ldots,d
 \end{cases}\quad,
\quad x\in \R^d,  \end{align*}
\begin{align*}
& \rho_1 (x , \theta) =
   \sigma(\scal{x}{\theta}{}{\R^{d+1}})\beta(\theta),
 \quad \rho_1(x,\cdot)\in\cc_0(\R^d\times\R) \ .
\end{align*}
Indeed,  let $\mu_0\in
  \mm(\Theta_0,\xx_1)=\bigoplus_{k=0}^d \R^{d+1}$. Then 
  $\mu_0=\sum_{k=o}^d  v_k \delta_k$ where $v_0,\ldots,
  v_d\in\R^{d+1} $, so that,  for all $x=(x_1,\ldots,x_d)\in\R^d$,
\[
x^{(1)}=v_0+\sum_{k=1}^d v_k x_k .
\]
Moreover, let $\mu_1\in \mm(\Theta_1,\xx_2)= \mm(\R^{d+1})$. Then 
$$
    f(x) = \int_{\R^{d+1}} \sigma(\sum_{k=1}^d
    \scal{v_k}{\theta}{}{\R^{d+1}}x_k +\scal{v_0}{\theta}{}{\R^{d+1}}
    ) \beta(\theta) d\mu_1(\theta) .
$$
Now, assuming that $v_0,\ldots,v_d$ are linearly independent,
let $ \Lambda:\R^{d+1}\to \R^{d+1} $ be the linear transformation such that ${\Lambda^\top} e_k= v_k$,
when $(e_k)_{k=0,\dots,d}$ is the canonical base of $\R^{d+1}$.
Then, setting $\theta=(w,b)\in\R^d\times\R$, $\beta'(\theta)=\beta(\Lambda^{-1}\theta)$,
and $\mu$ be the pushforward measure of $\mu_1$ by $\Lambda$, we obtain
$$
f(x) = \int_{\R\times\R^{d}}\sigma( \scal{w}{x}{}{\R^d} + b )
 \beta'(w,b)d\mu(w,b),
$$
These result in shallow neural networks with an uncountable number of hidden neurons.
In Section~\ref{sec:discrete-neural-rkbs}, inspired by this construction, we will introduce  deep networks where each hidden layer has countably many neurons.
\end{rmk}

\section{Representer Theorems for Deep Neural Networks}\label{sec:representer-theorems}

In this section, we state and prove representer theorems for deep integral and neural RKBS,
showing that common used deep neural networks can be seen as solutions of a variational problem.
We start by briefly recalling the basic supervised learning setting.

Let $\xx$ and $\yy$ be sets, called input and output space, respectively.
Consider a class $\hh$ of functions $ f : \xx \to \yy $, called \emph{hypothesis space},
a \emph{loss function} $ \mathcal{L} : \yy \times \yy \to [0,\infty) $,
and a \emph{penalty} $ \Phi : \hh \to [0,\infty) $.
Given $N$ samples
$$
 (x_i,y_i) \in \xx \times \yy , \qquad i = 1 , \dots , N  ,
$$
consider the \emph{regularized empirical risk minimization problem}
\begin{equation} \label{eq:erm}
 \min_{f \in \hh} \rr(f) + \Phi(f).
\end{equation}
Here,
$$
 \rr(f) =  \frac{1}{N} \sum_{i=1}^N \mathcal{L}( f(x_i) , y_i )
$$
is the empirical error associated to the loss function $\mathcal{L}$ and the training points $(x_i,y_i) \in \xx \times \yy, i = 1 , \dots , N $.
Representer theorems characterize the solutions of problems such as~\eqref{eq:erm}.

From now on, we fix a deep integral RKBS $\hh^{\rm deep}$ of depth $L$ with basis functions 
$
 \rho_\ell : \xx_\ell \times \Theta_\ell \to \R$,  as in Definition~\ref{dfn:deep-integral-rkbs}. 
To prove a representer theorem for  $\hh^{\rm deep}$
we will need the following fundamental assumption on the basis
functions $\rho_\ell$.

\begin{asp}\label{ass-ern}
The output space $\yy$ is a Hilbert space and, for each $\ell=1,\ldots,L$, the hidden layer $\xx_\ell$ is a separable 
reproducing kernel Hilbert space on the corresponding parameter space
$\Theta_\ell$ with a continuous reproducing kernel.
Furthermore, the
basis functions are of the form 
\begin{equation}\label{asp:x(th)}
 \rho_\ell(x,\th) = \widetilde{\rho}_\ell(x(\th),\th), \qquad x \in \xx_\ell, \quad \th \in \Th_\ell \ ,
\end{equation}
where $ \widetilde{\rho}_\ell : \R \times \Th_\ell \to \R $ is a
continuos function such that
$\rho_\ell(0,\cdot)\in\cc_0(\Theta_\ell)$, and 
\begin{equation}\label{asp:lip}
  | \rho_\ell(x,\th) - \rho_\ell(x',\th) | \le C_\ell | \langle x - x'
  , g_\ell(\th) \rangle_{\xx_\ell} | |\be_\ell(\th)| , \qquad x , x' \in \xx_\ell, \quad \th \in \Th_\ell \ ,
 \end{equation}
where $ C_\ell > 0 $, $ g_\ell \in \cc_b(\Th_\ell,\xx_\ell) $ and $
\beta_\ell \in \cc_0(\Th_\ell) $.
\end{asp}
The proof of the representer theorem for deep integral RKBS (\Cref{thm:representer}) relies on applying~\Cref{thm:bredies-carioni} to each hidden layer. \Cref{thm:bredies-carioni} requires a finite-dimensional setting. For the output layer, this condition is naturally met by assuming the output space is $\R^p$. However, since the hidden layers are infinite-dimensional, we must assume that each space $\xx_\ell$ is a reproducing kernel Hilbert space to reduce the problem to a finite-dimensional one, see~\eqref{xtilde}.
Assumption~\eqref{asp:lip} is essential to ensure that the evaluation functional $f_1(f_2(x))$ at $x\in\xx_{\ell-1}$ is jointly continuous in the pair $(f_1,f_2)\in \xx_{\ell}\times \xx_{\ell+1}$, as shown in~\Cref{lem:joint_weak*_cont}. This joint continuity, in turn, is crucial for establishing the existence of a minimizer of the regularized empirical risk minimization problem~\eqref{eq:erm}. We will see in Remark~\ref{rmk:assumption-satisfied} and Remark~\ref{rmk:ass_discrete_neural_rkbs} that such an assumption is easily satisfied by neural and discrete neural RKBS.

  \begin{rmk}\label{rmk:czero}
  In view of \Cref{ass-ern},
since $\xx_\ell$ is a reproducing kernel Hilbert
    space with a continuous reproducing kernel, $x(\cdot)$ is a
    continuous function, and hence $\rho_\ell$ is continuous on $\Theta_\ell$. Furthermore,
    given $x\in \xx_\ell$, for all $\th\in\Th_\ell$ we have
    \begin{alignat*}{1}
      |\rho_\ell(x,\theta)| &\leq |\rho_\ell(0,\theta)| 
      +|\rho_\ell(x,\theta) -\rho_\ell(0,\theta)| \\
& \leq C_\ell  |\scal{x}{g_\ell(\theta)}{}{\xx_{\ell}}|
      |\beta_\ell(\theta)| \\ & \leq C_\ell  \nor{x}{\xx_\ell}
      \sup_{\theta'\in\Theta_\ell}\nor{g_\ell(\theta')}{\xx_\ell}|\beta_\ell(\theta)|
    \end{alignat*}
    so that $\rho_\ell(x,\cdot)\in \cc_0(\Theta_\ell)$.
    Thus, the integral RKBS $\hh_\ell$, and therefore the associated deep integral RKBS, are well defined.
  \end{rmk}
  \begin{rmk}\label{riesz}
  Recall that, for each $\ell=1,\ldots, L$, the space $\Theta_\ell$ is second countable.
  Thus, for any $\mu\in\mathcal M(\Theta_\ell,\xx_{\ell+1})$, its variation $|\mu|$ is a finite measure, so $|\mu|$ is regular, and therefore $\mu$ is regular too \cite[Proposition 1]{dinculeanu1959representation}. Hence, taking into account that $\xx_{\ell+1}$ is a Hilbert space, so that $\xx_{\ell+1}'=\xx_{\ell+1}$, by a generalization of the Riesz representation theorem the Banach space $\mm(\Theta,\xx_\ell)$ can be identified with the dual of $\cc_0(\Theta_\ell,\xx_{\ell+1})$, see \cite{singer1957linear,ryan} for compact spaces,  
and \cite{carmeli2010vector} for second countable locally compact spaces.
It follows that $\mathcal M(\Theta_\ell,\xx_{\ell+1})$  can be endowed with the weak$^*$ topology, with respect to which the closed balls are compact.
\end{rmk}

We can now derive the representer theorem for deep integral RKBS.

\begin{thm}[Representer theorem for deep integral RKBS] \label{thm:representer}
Let $ \hh^{\rm deep}$  be a deep
integral RKBS  of depth $L$, induced by a Banach space $\mathcal{H}$, from the input space $\xx$
to the output space $\yy=\R^p$ satisfying~\Cref{ass-ern}. Assume that the loss
function $ \mathcal{L} $ is continuous in the first entry.
Then, there exist $ d_1 , \dots, d_{L+1} \in \N$ and, for all $\ell = 0 , \dots, L$,
\begin{alignat*}{1}
 & \theta^{(\ell)}_1,\ldots, \theta^{(\ell)}_{d_\ell}\in \Theta_\ell ,\\
 &   w^{(\ell+1)}_1, \ldots ,  w^{(\ell+1)}_{d_{\ell} } \in
    \widetilde{\xx}_{\ell+1}\subset
    \xx_{\ell+1}\quad\text{with}\quad \operatorname{dim}(\widetilde{\xx}_{\ell+1})\leq
    d_{\ell+1},
\end{alignat*}
such that
\[
f^{{\rm deep}}(x) = x^{(L+1)} \in \R^{p}, \qquad x\in\xx,
\]
with $x^{(L+1)}$ given by the recursive formula
\begin{equation}
\label{eq:6}
  \begin{cases}
  x^{(0)} = x \in \xx\\[5pt]
  x^{(\ell+1)} = \sum_{k=1}^{d_\ell} w^{(\ell+1)}_ k  \rho_\ell(x^{(\ell)},\theta^{(\ell)}_k) \in \widetilde{\xx}_{\ell+1},
  \qquad \ell = 0 , \dots , L,
 \end{cases}
\end{equation}
is a solution of the minimization problem
\begin{equation}\label{eq:initial_erm}
\min_{f^{{\rm deep}} \in \hh ^{{\rm deep}}} \rr(f^{{\rm deep}}) + \Phi(f^{{\rm deep}})  .
\end{equation}
Moreover, we have $ d_\ell \le N d_{\ell+1} $ for every $ \ell = 1 . \dots , L $, and
\begin{equation*}
\Phi(f^{{\rm deep}}) \le \sum_{\ell=0}^L \sum_{k=1}^{d_\ell} \| w^{(\ell+1)}_{ k} \|_{\xx_{\ell+1}}.
\end{equation*}
\end{thm}

\begin{proof}
By definition of $\hh^{\rm deep}$ and the complexity measure $\Phi$, the minimization problem is equivalent to 
\[
\min_{f \in \hh } \rr(f^{\rm deep}) + \nor{f}{\hh} ,
\]
where, for each $f=\oplus_{\ell=0}^L f_\ell\in\hh$, the function
$f^{\rm deep}:\xx_0\to \xx_{L+1}$ is the composition of
$f_{0},\ldots, f_{L}$ ($\xx_0=\xx$ and $\xx_{L+1}=\R^{p}$).
By construction, each $f_\ell$ is parameterized by some measure $\mu_\ell\in\mm_\ell = \mm(\Th_\ell,\xx_{\ell+1}) $, according to~\eqref{eq:integral}.
Moreover, by \eqref{eq:norma}, the minimization problem \eqref{eq:erm} is equivalent to 
\begin{equation} \label{eq:erm-mu}
\inf_{\mu \in \mm} \rr(f_\mu^{\rm deep}) + \sum_{\ell=0}^L \| \mu_\ell \|_{\TV} =: \inf_{\mu \in \mm}\cS(\mu) ,
\end{equation}
where $\mm$ is the Banach space $\bigoplus_{\ell=0}^L
  \mm(\Th_\ell,\xx_{\ell+1}) $ and, if $\mu=\mu_0\oplus\cdots\oplus\mu_L\in\mm$, $f_\mu^{\rm deep}$ is the composition of $f_{\mu_0},\ldots, f_{\mu_L}$.

 Fix a $ \nu \in \mm $, and let $ R = \cS(\nu)$. Then  \eqref{eq:erm-mu} is equivalent to 
\begin{equation} \label{eq:erm-R}
 \inf_{\mu \in \prod_{\ell=0}^L B_{\mm_\ell}(R)} \cS(\mu)  .
\end{equation}
Indeed, if $\mu$ is outside $ \prod_{\ell=0}^L B_{\mm_\ell}(R) $, then for some $\ell=0,
\ldots,L$, $\nor{\mu_\ell}{\text{TV}}>R $, so that 
$$
 \cS(\mu) = \rr(f_\mu) + \sum_{\ell=0}^L \| \mu_\ell \|_{\TV}
 \ge \rr(f_\mu) + \cS(\nu)
 \ge \cS(\nu),
$$
which proves the equivalence. We now prove the existence of a
minimizer of~\eqref{eq:erm-R}.

{\it Existence of a minimizer.}
Assumption~\ref{ass-ern} along with \Cref{rmk:czero} ensures that, for
  any $\ell=0,\ldots,L$, the
  pair $(\rho_{\ell},\rho_{\ell+1})$ of basis functions  satisfies the
  conditions of~\Cref{lem:joint_weak*_cont}.
  Hence, taking into
  account~\Cref{rmk:evaluation}, it follows that , for all $ x \in \xx $, the map
$$
(\mu_0,\dots,\mu_L) \mapsto f_{\mu_L} \circ \cdots \circ f_{\mu_0} (x)
$$
is \emph{jointly} continuous from $ \prod_{\ell=0}^L B_{\mm_\ell}(R)
$, endowed with the product topology induced by the weak$^*$ topology
of each $B_{\mm_\ell}(R)$,  to $\yy$. Since $\mu_\ell \mapsto \nor{\mu}{\TV}$ is weakly$^*$ continuous, 
the map $ \mu \mapsto \cS(\mu) $ is also continuous.
Moreover, thanks to the Banach--Alaoglu theorem,
the product $ \prod_{\ell=0}^L B_{\mm_\ell}(R) $ is weakly$^*$ compact.
Hence, by the extreme value theorem, the problem \eqref{eq:erm-R} has at least a minimizer. Now we prove the representer theorem showing that a minimizer can be taken to have a finite width at every layer. Figure~\ref{fig:representer_thm} summarizes the proof structure of the representer theorem. 

{\it Representer theorem.} Let $ \mu^* $ be any such solution.
Let $ x_i^{(0)} = x_i $ and $ x_i^{(\ell+1)} = f_{\mu^*_\ell} (x_i^{(\ell)}) $ for $ \ell = 0 , \dots , L $ and $i=1,\ldots,N$.
Then, in view of \Cref{lem:min-interpol},
a solution to \eqref{eq:erm-R} can be found by solving the following interpolation problems
for all $ \ell = 0 , \dots , L $:
\begin{equation} \label{eq:interpol}
 \inf_{\mu_\ell \in \mm_\ell} \| \mu_\ell \|_{\TV} \quad \text{subject to} \quad f_{\mu_\ell} (x_i^{(\ell)}) = x_i^{(\ell+1)} \quad i = 1 , \dots , N .
\end{equation}

{\it Case $ \ell = L $.} Let us start from $ \ell = L $.
We want to apply \Cref{thm:bredies-carioni}.
 To this end, let $ \mathcal{U} = \mm_L $ endowed with the weak$^*$ topology.
 We define $ \aa : \mathcal{U} \to \R^{N \times d_{L+1}} $ by
 $$
  \aa \mu = [ f_{\mu_{L}}(x_i^{(L)}) ]_{i=1,\dots,N} .
 $$
 Then $\aa$ is a surjective continuous linear operator from $\mathcal{U}$ onto $ H = \ran\aa$,
 with $ \dim(H) \le Nd_{L+1} $.
 Moreover, the norm $ G = \|\cdot\|_{\TV} $ is coercive on $\mathcal{U}$.
 Indeed, by the Banach--Alaoglu theorem, the balls $ B_{\mm_L}(r) $ are weakly$^*$ compact for every $ r > 0 $.
 We define $ F : H \to [0,\infty] $ by
 $$
  F(h) = \begin{cases}
   0 & h_i = x_i^{(L+1)} \text{ for all $i=1,\dots,N$} \\
   \infty & h_i \ne x_i^{(L+1)} \text{ for some $i=1,\dots,N$}  .
  \end{cases}
 $$
The function $F$ is the indicator function associated to the singleton $\{ (x_1^{(L+1)},\ldots,x_N^{(L+1)})\}$, so that it is convex, coercive and lower semi-continuous.
 Therefore, we can apply \Cref{thm:bredies-carioni} to derive that \eqref{eq:interpol} has a solution of the form
 $$
  \widetilde{\mu}_L = \sum_{k=1}^{d_L} c_k u_k,
 $$
 for some $ d_L \le N d_{L+1} $, $ c_k > 0 $ and $ u_k \in \Ext(B_{\mm_L}(1)) $.
By \Cref{lem:Ext(B)}, for each $k$ there are $ y_k \in \Ext(B_{\xx_{L+1}}(1)) $ and $ \theta^{(L)}_k \in \Th_L $ such that
 $$
  u_k = y_k \cdot \delta_{\theta^{(L)}_k}  .
 $$
 Thus, defining $ w^{(L+1)}_{\cdot k} = c_k y_k \in {\xx_{L+1}}$, we get
 $$
  \widetilde{\mu}_L = \sum_{k=1}^{d_L} w^{(L+1)}_{ k} \delta_{\theta^{(L)}_k}  .
 $$
 Thanks to \Cref{ass-ern}, for all $x\in\xx_L$ we have
\begin{alignat}{1}
 x^{(L+1)}= f _{\widetilde{\mu}_L} (x) & = \sum_{k=1}^{d_L} w^{(L+1)}_{ k}
\widetilde{\rho}_L(x(\theta_k^{(L)}),\th_k^{(L)}).
\label{eq:7}
\end{alignat}
Thus, \eqref{eq:6} holds true for $\ell=L$ with
$\widetilde{\xx}_{L+1}=\xx_{L+1}=\R^{p}$ and $d_{L+1}=p$. 

{\it Case $ \ell = L-1$.} Now we consider $ \ell = L-1 $.  Let  $K^{(L)}$ be the reproducing kernel of $\xx_L$ and set 
\begin{equation}\label{xtilde}
\widetilde{\xx}_L=\operatorname{span}\{
  K^{(L)}(\cdot,\theta^{(L)}_k):k=1,\ldots,d_L\}\subset\xx_{L}, \qquad
  \dim{\widetilde{\xx}_L}\leq d_L,    
\end{equation}
and $P_L$ the corresponding orthogonal projection from $\xx_L$ onto
$\widetilde{\xx}_L$.  By~\eqref{eq:7},  it is clear that 
$f _{\widetilde{\mu}_L}(x)=f _{\widetilde{\mu}_L}(P_Lx)$ for all
$x\in\xx_L$. Hence, since $\nor{P_L\mu}{\TV}\leq  \nor{\mu}{\TV}$ for
all $\mu\in \mm(\Th_{L-1},\xx_L)$, the direct sum of measures
\[
(\mu^*)'=\mu^*_0 \oplus\ldots\oplus \mu^*_{L-2}\oplus P_L\mu^*_{L-1}\oplus
\widetilde{\mu}_L  
\]
is a minimizer of~\eqref{eq:erm-R}. By \Cref{rmk:freedom}, for each $i=1,\ldots,N$,  we can
replace  $x^{(L)}_i=f_{\mu^*_{L-1}}(x_i^{(L-1)})$
with $P_L x_i^{(L)}=f_{P_L\mu^*_{L-1}}(x_i^{(L-1)})$, so that  \eqref{eq:interpol}
for $\ell=L-1$ reads as
$$
   \inf_{\mu_{L-1} \in \mm(\Th_{L-1},\widetilde{\xx}_L)} \| \mu_{L-1} \|_{\TV} \quad \text{subject to} \quad f_{\mu_{L-1}} (x_i^{(L-1)}) = {P_Lx_i^{(L)}} \quad i = 1 , \dots , N .
 $$ 
Since $\widetilde{\xx}_L$ is finite dimensional,
we can use again \Cref{thm:bredies-carioni}, together with \Cref{lem:Ext(B)}, as in the previous step ($\ell=L$).
In particular, this time we have $ \dim(H) \le N d_L $.
Thus, we find a solution of the form
 $$
  \widetilde{\mu}_{L-1} = \sum_{k=1}^{d_{L-1}} w^{(L)}_{ k} \delta_{\theta^{(L-1)}_k},
 $$
 for some $ d_{L-1} \le N d_L $ and $w^{(L)}_{1}, \ldots,
   w^{(L)}_{d_{L-1}}\in\widetilde{\xx}_{L}$.  Now, for all $x\in\xx_{L-1}$ we have
\[
f _{\widetilde{\mu}_{L-1}} (x)=\sum_{k=1}^{d_{L-1}}
w^{(L)}_{k}  \widetilde{\rho}_{L-1}(x(\theta_k^{(L-1)}),\th_k^{(L-1)}) .
\]
 Iterating the argument,~\eqref{eq:6} holds true for $ \ell
 =1,\ldots,L$.  
 
 {\it Case $\ell=0$.} For the last step ($\ell=0$),  the minimization problem 
 is
\[
\inf_{\mu_0 \in \mm_0} \| \mu_0\|_{\rm TV} \quad \text{subject to} \quad f_{\mu_0} (x_i) = P_1x_i^{(1)}\in\widetilde{\xx}_1 \quad i = 1 , \dots , N  ,
\]
which, as above, admits a solution $\widetilde{\mu}_0$ such that
\[
f _{\widetilde{\mu}_0} (x)  = \sum_{k=1}^{d_0} w^{(1)}_{ k}
 \rho_0(x,\th_k^{(0)})  .
\]
Then, $\widetilde{\mu}=\widetilde{\mu}_0 \oplus\ldots\oplus \widetilde{\mu}_{L-1}\oplus
\widetilde{\mu}_L$ is a solution of \eqref{eq:erm-R} and, consequently, $f_{\widetilde{\mu}}^{\rm deep}=f_{\widetilde{\mu}_0}\circ\cdots\circ f_{\widetilde{\mu}_L}$ is a solution of the initial minimization problem \eqref{eq:initial_erm}.
The bound on $d_\ell$ is also clear by iteration.
 For the bound on $ \Phi(f^{{\rm deep}}) $, by definition we have
 $$
 \Phi(f^{{\rm deep}})\leq \sum_{\ell=0}^L \| f_{\widetilde{\mu}_\ell} \|_{\hh_\ell}
  \le \sum_{\ell=0}^L \| \widetilde{\mu}_\ell \|_{\rm TV} ,
 $$
 where, once again by \Cref{thm:bredies-carioni},
 $$
  \| \widetilde{\mu}_\ell \|_{\TV} = \sum_{k=1}^{d_\ell} c_k  .
 $$
 But since $ y_k \in \Ext( B_{\widetilde{\xx}_{\ell+1}}(1) ) $,
 we have $ \| y_k \|_{\widetilde{\xx}_{\ell+1}} = 1 $, whence, identifying again $\widetilde{\xx}_{\ell+1}$ with a subspace of $\xx_{\ell+1}$,
 $$
  \sum_{k=1}^{d_\ell} c_k = \sum_{k=1}^{d_\ell} \| c_k y_k \|_{{\xx_{\ell+1}}}
  = \sum_{k=1}^{d_\ell} \| w^{(\ell+1)}_{ k} \|_{\xx_{\ell+1}}  ,
 $$
 which concludes the proof.
 \begin{figure}
$$
 \xymatrix{
 \xx \ar[r]^{f_{\mu_0^*}} & \xx_1 \ar[r]^{f_{\mu_1^*}} & \cdots \ar[r]^{f_{\mu_{L-2}^*}} & \xx_{L-1} \ar[r]^{f_{\mu_{L-1}^*}} & \xx_{L} \ar[r]^{f_{\mu_{L}^*}} & \R^p}
 $$
 $$
  \xymatrix{
 \xx  \ar[r]^{f_{\mu_0^*}} & \xx_1 \ar[r]^{f_{\mu_1^*}} & \cdots \ar[r]^{f_{\mu_{L-2}^*}} & \xx_{L-1} \ar[r]^{f_{P_L\mu_{L-1}^*}} & \widetilde{\xx}_{L} \ar[r]^{f_{\widetilde{\mu}_{L}}} & \R^p}
$$
 $$
  \xymatrix{
 \xx  \ar[r]^{f_{\mu_0^*}} & \xx_1 \ar[r]^{f_{\mu_1^*}} & \cdots \ar[r]^{f_{P_{L-1}\mu_{L-2}^*}\quad} & \widetilde{\xx}_{L-1} \ar[r]^{f_{\widetilde{\mu}_{L-1}}} & \widetilde{\xx}_{L} \ar[r]^{f_{\widetilde{\mu}_{L}}} & \R^p}
$$
$$
\vdots
$$
 $$
  \xymatrix{
\xx  \ar[r]^{f_{\widetilde{\mu}_0}} & \widetilde{\xx}_1 \ar[r]^{f_{\widetilde{\mu}_1}} & \cdots \ar[r]^{f_{\widetilde{\mu}_{L-2}}} & \widetilde{\xx}_{L-1} \ar[r]^{f_{\widetilde{\mu}_{L-1}}} & \widetilde{\xx}_{L} \ar[r]^{f_{\widetilde{\mu}_{L}}} & \R^p}
$$
\caption{Proof structure of Theorem~\ref{thm:representer}.}
\label{fig:representer_thm}
\end{figure}
\end{proof}

As established in Corollary~\ref{cor:representer}, in the special case of a neural RKBS, the minimizer $f^{\rm deep}$ takes the form of a deep neural network with finite width at each hidden layer according to Definition~\ref{def:nn}.
\begin{rmk}\label{rmk:assumption-satisfied}
    For neural RKBS,
\ref{asp:x(th)} is satisfied with $ \widetilde{\rho}_\ell(t,n) = \si(t+c(\theta))\be(\theta)$ for all $ \ell =1 , \dots, L$. Furthermore, if $\si$ is Lipschitz with Lipschitz constant $C_\si$, \ref{asp:lip} is satisfied with $
C_\ell = C_\si $ and $ g_\ell(\theta) = \de_{\theta} $ for all $ \ell = 1 , \dots , L $. Indeed, for all $ x , x' \in\xx_\ell$ and $ \theta\in\Theta_\ell$, we have that
 \begin{align*}
  | \rho_\ell(x,\theta) - \rho_\ell(x',\theta) | &= |\sigma(x(\theta)+c_\ell(\theta))-\sigma(x'(\theta)+c_\ell(\theta))||\be_\ell(\theta)|\\&\le C_\si | x(\theta)-x'(\theta)||\be_\ell(\theta)| .
  \end{align*}
\end{rmk}

\begin{cor}[Representer theorem for neural RKBS] \label{cor:representer}
Under the assumptions of \Cref{thm:representer},
if $ \hh^{\rm deep} $ is a neural RKBS,
then there exist
\begin{alignat*}{1}
 & d_1 , \dots, d_L \in \N, \qquad d_\ell \le N d_{\ell+1}, \\
 & W^{(\ell)} \in \R^{d_\ell\times d_{\ell-1}},
 \qquad
 b^{(\ell)} \in \R^{d_{\ell}}, \qquad \ell = 1 , \dots, L+1,
\end{alignat*}
such that
\begin{equation}
f^{\rm deep}(x) = x^{(L+1)} \in \R^{p}, \qquad x\in\xx,
\end{equation}
with $x^{(L+1)}$ given by recursive formula
\begin{equation}
      \begin{cases}
        x^{(1)} = W^{(1)} x + b^{(1)} \\
        x^{(\ell+1)} = W^{(\ell+1)} \sigma( x^{(\ell)} ) +
        b^{(\ell+1)}, \qquad \ell=1,\ldots,L, 
      \end{cases}
    \end{equation}
is a solution of the minimization problem
\begin{equation}\label{eq:ermp_neural}
\min_{f^{\rm deep} \in \hh ^{\rm deep}} \rr(f^{\rm deep}) + \Phi(f^{\rm deep})   .
\end{equation}
\end{cor}
\begin{proof}
Revisiting the proof of \Cref{thm:representer},
for all $\ell=1,\dots,L+1$, we can choose a basis
    $\{e^{(\ell)}_k\}_{k=1}^{\operatorname{dim}\xx_\ell}$ of $\xx_{\ell}$
    such that
    $$
    \widetilde{\xx}_{\ell} \subset  \operatorname{span} \{e^{(\ell)}_1,\ldots,e^{(\ell)}_{d_\ell} \} ,
    $$
so that the
    elements
    $w^{(\ell)}_1,\ldots,
    w^{(\ell)}_{d_{\ell-1}}\in\widetilde{\xx}_\ell$ can be identified
    with vectors in $\R^{d_\ell}$ and collected in a
    $d_{\ell}\times d_{\ell-1}$ matrix
    \[
      U^{(\ell)} = \left(
        \begin{array}{c|c|c}
          w^{(\ell)}_1 \beta_{\ell-1}(\theta_1^{(\ell-1)})& \cdots &
                                                                   w^{(\ell)}_{d_{\ell-1}}\beta_{\ell-1}(\theta_{d_{\ell-1}}^{(\ell-1}) )
        \end{array}
      \right)  ,
    \]
    where $\beta_0=1$.  Similarly, for all $\ell=1,\ldots,L$, the elements
    $K^{(\ell)} (\cdot,\theta^{(\ell)}_1),\ldots, K^{(\ell)}
    (\cdot,\theta^{(\ell)}_{d_{\ell}})\in \widetilde{\xx}_{\ell} \simeq
    \R^{d_\ell}$ define a $d_{\ell}\times d_{\ell}$ matrix
    \[
      V^{(\ell)} = \left(
        \begin{array}{c|c|c}
          K^{(\ell)}  (\cdot,\theta^{(\ell)}_1)^\top & \cdots &  K^{(\ell)} (\cdot,\theta^{(\ell)}_{d_{\ell}})^\top
        \end{array}
      \right)  ,
    \]
    and the offsets
    $c_\ell(\theta^{(\ell)}_1),
    \ldots,c_\ell(\theta^{(\ell)}_{d_{\ell}})\in\R$ a vector
    \[c^{(\ell)} = \left(c_\ell(\theta^{(\ell)}_1), \ldots,
        c_\ell(\theta^{(\ell)}_{d_{\ell}}) \right)^\top \in \R^{d_\ell}.  \]
    Hence, \eqref{eq:6} reads as
    \[
      \begin{cases}
        x^{(0)} = (x(\theta_1^{(0)}),\ldots, x(\theta_{d_0}^{(0)}))^\top \\
        x^{(1)} = U^{(1)} x^{(0)} \\
        x^{(\ell+1)} = U^{(\ell+1)} \left(\sigma (
          V^{(\ell)} x^{(\ell)} + c^{(\ell)}) \right),\qquad
        \ell=1,\ldots, L.
      \end{cases}
    \]
For all
    $\ell=1,\ldots,L+1$, define the matrix
    \[ W^{(\ell)} = V^{(\ell)}U^{(\ell)} \in \R^{d_\ell\times d_{\ell-1}}, \]
    where $V^{(L+1)}=\operatorname{Id}_{d_{L+1}\times d_{L+1}}$, and the vector
    \[
      b^{(\ell)}= c^{(\ell)}\in \R^{d_{\ell}},
    \]
    where $c^{(L+1)}=0\in\R^{p}$.
    Up to redefining the points $x^{(\ell)}$ for $\ell=1,\ldots, L$, a
    solution $f^{\rm deep}$ of the empirical risk minimization problem \eqref{eq:ermp_neural} is given by $f^{\rm deep}(x)=x^{(L+1)}\in\R^p$ with $x^{(L+1)}$ given by the recursive formula 
       \[
      \begin{cases}
        x^{(0)} = (x(\theta_1^{(0)}),\ldots, x(\theta_{d_0}^{(0)}))^\top \\
        x^{(1)} = W^{(1)} x^{(0)}+b^{(1)} \\
        x^{(\ell+1)} = W^{(\ell+1)}\sigma (
          x^{(\ell)})+b^{(\ell+1)},\qquad
        \ell=1,\ldots, L,
      \end{cases}
    \]
    which concludes the proof.
\end{proof}

\section{Discrete Neural RKBS}\label{sec:discrete-neural-rkbs}

 Next, we construct a particular instance of neural RKBS with a countably infinite number of neurons per hidden layer, which we call \emph{discrete} neural RKBS, see Figure~\ref{fig:discrete-neural-rkbs}. This further specialization to these spaces allows for a more explicit characterization of the complexity measure $\Phi$ in the corresponding representer theorem.
 
In the following formulation, differently from the one in \Cref{bartolucci}, the spaces $\Theta_\ell$ of the hidden layers play the role of index spaces for the  parameters,
and they do not directly correspond to the spaces where parameters live.

\begin{dfn}[Discrete Neural RKBS]\label{dfn:discrete-neural-rkbs}
Fix the parameter spaces $\Theta_\ell$ and the layer spaces $\xx_\ell$ as follows
\begin{alignat*}{3}
 &  \Theta_0 = \{ 0 , \dots, d \}, \qquad && \xx_0 = \ell^2(\{ 1 ,
 \dots, d \}) = \R^d, \qquad &&  \\
 & \Theta_\ell = \N, \qquad &&
 \xx_{\ell} = \ell^2(\N), \qquad &&\ell = 1 , \dots , L, \\
 & \Theta_{L+1} = \{1,\dots,p\}, \qquad && \xx_{L+1} =
 \ell^2(\{1,\dots,p\}) = \R^p, \qquad  && \ell=L+1,
\end{alignat*}
and let $ \sigma : \R \to \R $ be a Lipschitz activation function such that $\sigma(0)=0$. 
Then, set
\begin{alignat*}{2}
 & \rho_0 (x , n) =
 \begin{cases}
   1 & \hspace{29pt} n=0 \\
   x_n & \hspace{29pt} n=1,\ldots,d 
 \end{cases}\quad,
\qquad && x\in \R^d  \\
 & \rho_\ell (x , n) =
 \begin{cases}
   1 & n=0 \\
   \sigma(x_{n-1}) & n\geq 1 
 \end{cases}\quad, && x\in\ell^2(\N),
\qquad \ell = 1 , \dots , L \ .
\end{alignat*}


We call the resulting space $\hh^{\rm deep}$ a \emph{discrete} neural RKBS.
\end{dfn}

\begin{rmk} Some observations on Definition~\ref{dfn:discrete-neural-rkbs} are in order. First, note that  $ \xx_0 $ can be thought of as a function space on $ \Th_0 $
by putting $ x ( 0 ) = 1 $ for all $ x \in \xx_0 $. Second, the Lipschitzianity of $\sigma$ implies  that 
\begin{equation}
  \label{eq:2a}
  \sigma(x)\in\ell^2(\N),  \qquad \text{for all }
  x\in\ell^2(\N),
\end{equation}
where $\sigma$ applies on sequences component by component.
Indeed, if $C_\si$ denotes the Lipschitz constant of $\sigma$, we have 
\begin{align*}
     |\sigma(x_n)|   &= |\sigma(x_n)-\sigma(0)|\leq C_\si  |x_n-0|=C_\si|x_n| .
\end{align*}
In particular, this implies that  
\begin{equation}
 \label{eq:2}
  \lim_{n\to\infty} \sigma(x_n)    = 0, \qquad \text{for all }
 x\in\ell^2(\N) \ .
\end{equation}

Finally, assuming $\sigma(0)=0$ is not restrictive. In fact, if $\sigma(0)\ne0$, the choice of the activation function $\sigma'=\sigma-\sigma(0)$ simply implies a scaling of the offsets. 
\end{rmk}

\subsection*{General form of discrete neural RKBS functions}
An element $f^{\rm deep}$ of the neural RKBS $\hh^{\rm deep}$
is a composition of $L+1$ integral functions $f_0$,\ldots,$f_L$, where each $f_\ell$ is defined via a measure $\mu_\ell$ by $ f_\ell = f_{\mu_\ell} $. We now derive an explicit expression for elements in a discrete neural RKBS, showing that these functions correspond to infinite-width neural networks with a countably infinite number of neurons per hidden layer, see Figure~\ref{fig:discrete-neural-rkbs}.
\begin{figure}
$$
 \xymatrix{
 \R^{d} \ar@/_1.5pc/[rrrr]_{f^{\rm deep}} \ar[r]^{f_0} & \ell^2(\N) \ar[r] & \cdots \ar[r] & \ell^2(\N) \ar[r]^{f_L} & \R^p
 }
$$
\begin{tikzpicture}[->,shorten >=1pt,>=stealth]

\tikzstyle{neuron} = [circle, draw=black, fill=white,
                      minimum size=20pt, inner sep=0pt]
\tikzstyle{dots}   = [minimum size=10pt, inner sep=0pt]

\node[neuron] (I1) at (0,  2)  {$x_1$};
\node[neuron] (I2) at (0,  1)  {$x_2$};
\node[dots]   (Id) at (0,  0)  {$\vdots$};
\node[neuron] (Idn)at (0, -1)  {$x_d$};

\node[dots]   (H1top)  at (2,  2.5)  {$\vdots$};
\node[neuron] (H1A)    at (2,  1.5)  {};
\node[dots]   (H1mid)  at (2,  0.5)  {$\vdots$};
\node[neuron] (H1B)    at (2, -0.5)  {};
\node[dots]   (H1bot)  at (2, -1.5)  {$\vdots$};

\node[dots]   (H2top)  at (4,  2.5)  {$\hdots$};
\node[dots] (H2A)    at (4,  1.5)  {$\hdots$};
\node[dots]   (H2mid)  at (4,  0.5)  {$\hdots$};
\node[dots] (H2B)    at (4, -0.5)  {$\hdots$};
\node[dots]   (H2bot)  at (4, -1.5)  {$\hdots$};

\node[dots]   (H3top)  at (6,  2.5)  {$\vdots$};
\node[neuron] (H3A)    at (6,  1.5)  {};
\node[dots]   (H3mid)  at (6,  0.5)  {$\vdots$};
\node[neuron] (H3B)    at (6, -0.5)  {};
\node[dots]   (H3bot)  at (6, -1.5)  {$\vdots$};

\node[neuron] (O1)   at (8,  2)  {$y_1$};
\node[neuron] (O2)   at (8,  1)  {$y_2$};
\node[dots]   (Od)   at (8,  0)  {$\vdots$};
\node[neuron] (Op)   at (8, -1)  {$y_p$};

\foreach \i in {I1,I2,Id,Idn}
  \foreach \j in {H1top,H1A,H1mid,H1B,H1bot}
    \draw (\i) -- (\j);

\foreach \i in {H1top,H1A,H1mid,H1B,H1bot}
  \foreach \j in {H2top,H2A,H2mid,H2B,H2bot}
    \draw (\i) -- (\j);

\foreach \i in {H2top,H2A,H2mid,H2B,H2bot}
  \foreach \j in {H3top,H3A,H3mid,H3B,H3bot}
    \draw (\i) -- (\j);

\foreach \i in {H3top,H3A,H3mid,H3B,H3bot}
  \foreach \j in {O1,O2,Od,Op}
    \draw (\i) -- (\j);

\end{tikzpicture}
\caption{Architecture of discrete neural RKBS functions.}
\label{fig:discrete-neural-rkbs}
\end{figure}

{\it Case $\ell=0$.} From layer $0$ to layer $1$, we have
$$
\mu_0 \in \mm(\{0,\dots,d\},\ell^2(\N))  ,
$$
so that 
\[
  \quad \mu_0  = \sum_{m=0}^d  w^{(1)}_{m} \delta_m, 
\]
for a family of $d+1$ vectors $ w^{(1)}_{0},\ldots w^{(1)}_{d} \in \ell^2(\N) $. Let
  $b^{(1)} = w^{(1)}_{0} \in \ell^2(\N)$ and define the bounded
  operator 
\[ W^{(1)} \in B(\R^d, \ell^2(\N)), \qquad
 W^{(1)} x  = \sum_{m=1}^d  w^{(1)}_{m} x_m  .
\]
Then, for $ x \in \R^d $,  the function $f_0:\R^d\to\ell^2(\N)$ is
\[
f_0(x) =  \sum_{m=0}^d  w^{(1)}_{m}  \rho_0(x,m)=
w^{(1)}_0+\sum_{m=1}^d  w^{(1)}_{m} x_m =
W^{(1)} x + b^{(1)} \in \ell^2(\N)  ,
\]
and the scalar components of  $f_0$ are 
\[
f_0(x)_n 
 =   \langle x , w^{(1)}_{n\cdot} \rangle_{\R^d} +    b^{(1)}_n, \qquad n\in\N,
\]
where $w^{(1)}_{n\cdot}\in\R^d$  with $w^{(1)}_{n m } = (w^{(1)}_m)_n$.


{\it Case $ \ell = 1 , \dots , L-1 $}. From layer $\ell$ to layer $\ell+1$ we have
$$
 \mu_\ell \in \mm(\N,\ell^2(\N))  ,
$$
so that
\[
\mu_\ell  = \sum_{m=0}^\infty  w^{(\ell+1)}_{m} \delta_m,
\]
for a countable family of vectors $ w^{(\ell+1)}_{0},\ldots,
w^{(\ell+1)}_{m},\ldots \in \ell^2(\N) $ such that
  \begin{equation}
\nor{\mu_\ell}{ \text{TV}}=\sum_{m=0}^\infty  \nor{w^{(\ell+1)}_{m}}{\ell^2(\N)}
<\infty\label{eq:1} .
\end{equation}
 As before, set 
  $b^{(\ell+1)} = w^{(\ell+1)}_{0} \in \ell^2(\N)$ and 
\[ W^{(\ell+1)} \in B(\ell^2(\N), \ell^2(\N)),\qquad
 W^{(\ell+1)} x  = \sum_{m=1}^{+\infty}  w^{(\ell+1)}_{m} x_{m-1}  ,
\]
where the series converges absolutely in $\ell_2(\N)$ due
to~\eqref{eq:1}. 
Hence, for $ x \in \ell_2(\N)$, 
  \begin{alignat*}{1}
    f_\ell(x) & = \sum_{m=0}^{+\infty} w^{(\ell+1)}_{m} \rho_\ell(x,m)=
    w^{(\ell+1)}_0+\sum_{m=1}^{+\infty} w^{(\ell+1)}_{m}
    \sigma(x_{m-1}) \\
& = W^{(\ell+1)} \left(\sigma(x)\right) +
    b^{(\ell+1)} \in \ell^2(\N)  ,
  \end{alignat*}
where $\sigma(x)\in\ell^2(\N)$ by~\eqref{eq:2a}. The
component of $f_\ell(x)$ are 
\[
f_\ell (x)_n 
 =  \scal{\sigma(x)}{ w^{(\ell+1)}_{n\cdot}}{}{\ell^2}
 +    b^{(\ell+1)}_n, \qquad n\in\N ,
\]
where $w^{(\ell+1)}_{n\cdot}\in\ell^2(\N)$  and $w^{(\ell+1)}_{n m } = (w^{(\ell+1)}_m)_n$.

{\it Case $\ell=L$.} Finally, from layer $L$ to layer $L+1$, we have
$$
 \mu_L \in \mm(\N,\R^p) ,
$$
so that
\[
\mu_L  = \sum_{m=0}^\infty  w^{(L+1)}_{m} \delta_m,
\]
for a countable family of vectors $ w^{(L+1)}_{0},\ldots,
w^{(L+1)}_{m},\ldots \in \R^p$ such that
  \begin{equation}
\sum_{m=0}^\infty  \nor{w^{(\ell+1)}_{m}}{\R^p}
<+\infty\label{eq:1a} .
\end{equation}
 As before, set 
  $b^{(L+1)} = w^{(L+1)}_{0} \in \ell^2(\N)$ and 
\[ W^{(L+1)} \in B(\ell^2(\N), \R^p), \qquad
 W^{(L+1)} x  = \sum_{m=1}^{+\infty}  w^{(L+1)}_{m} x_{m-1}  ,
\]
where the series  converges absolutely in $\R^p$ due
to~\eqref{eq:1a}. 
Hence, for $ x \in \ell_2(\N)$, 
  \begin{alignat*}{1}
    f_L(x) & = \sum_{m=0}^{+\infty} w^{(L+1)}_{m} \rho_L(x,m)=
    w^{(L+1)}_0+\sum_{m=1}^{+\infty} w^{(L+1)}_{m}
    \sigma(x_{m-1}) \\
& = W^{(L+1)} \left(\sigma(x)\right) +
    b^{(L+1)} \in \ell^2(\N)  ,
  \end{alignat*}
where $\sigma(x)\in\ell^2(\N)$ by~\eqref{eq:2a}. The
component of $f_L(x)$ are 
\begin{equation}
f_L (x)_n 
 =  \scal{\sigma(x)}{ w^{(L+1)}_{n\cdot}}{}{\ell^2} +
 b^{(L+1)}_n, \qquad n=1,\ldots,p ,\label{eq:3}
\end{equation}
where $w^{(L+1)}_{n\cdot}\in\ell^2(\N)$  and $w^{(L+1)}_{n m } = (w^{(L+1)}_m)_n$.


By iteration $ x^{(\ell+1)} = f_\ell(x^{(\ell)}) $,
we  obtain
$$
\begin{cases}
 x^{(0)} = x & \in \R^d \\[5pt]
 x^{(1)} =  W^{(1)} x^{(0)} + b^{(1)} & \in \ell^2(\N) \\[5pt]
 x^{(\ell+1)} = W^{(\ell+1)} \left(\si(x^{(\ell)})\right) + b^{(\ell+1)}
 & \in \ell^2(\N) \\[5pt]
 x^{(L+1)} = W^{(L+1)} \left(\si(x^{(L)})\right) + b^{(L+1)} & \in \R^p \ .
\end{cases}
$$
We stress that  the width of the $L$ hidden layers
$ \ell = 1 , \dots , L $ is infinite and countable (the neurons are
parameterized by $\N$),  while input and output layers $\ell=0$ and $\ell=L+1$ have fixed finite widths $d$ and $p$, respectively.
Also note that infinite-width neural networks generalize finite-width neural networks, 
where the  inner layers are generated by infinite-rank operators.

We can visualize the shallow ($L=1$), and the simplest non-shallow case ($L=2$), 
considering a non-iterative expression.
For $ L = 1 $, we have
$$
 f(x) = W^{(2)} \left(\si(W^{(1)} x + b^{(1)})\right) + b^{(2)} \ .
$$
In the case of $ L = 2 $ hidden layers, we can write
$$
 f(x) = W^{(3)} \left(\si\left(W^{(2)} \left(\si(W^{(1)} x + b^{(1)}) \right)+ b^{(2)}\right)\right) + b^{(3)} \ .
$$

Comparing our construction to the one proposed in \cite{doi:10.1137/21M1418642},
we remark that our networks do not have any rank constraint,
in the sense that every hidden layer has infinte width. 


\subsection*{Finite form of discrete neural RKBS functions}

We now show that the neural functions defined in~\ref{def:nn}
  correspond to measures $\mu_1,\ldots,\mu_L$ having finite
  support. Indeed,  under this assumption,  for each $\ell=1,\ldots, L$,
\[
\mu_\ell = b^{(\ell+1)} \delta_0 + \sum_{k=1}^{d_\ell}
w^{(\ell+1)}_{k}  \delta_{m^{(\ell)}_{k}},
\]
for some $m^{(\ell)}_{1},\ldots,
m^{(\ell)}_{d_\ell} \in\N\setminus\{0\}$ and some $w^{(\ell+1)}_{1},\ldots,
w^{(\ell+1)}_{k} $ that are in $\ell_2(\N)$ if $\ell<L$, and in $\R^p$ if $\ell=L$.

We define $f_\ell$ starting
from the last layer. 
Since the support of $\mu_L$ is
$\{0,m^{(L)}_1,\ldots,m^{(L)}_{d_L}\}$, 
by~\eqref{eq:3} $f_L$ depends only on the variables
$m^{(L)}_1,\ldots, m^{(L)}_{d_L}$, so that we can regard $f_L$ as a
function from $\R^{d_L}$ to $\R^p$ given by
\[ f_L(x)= W^{(L+1)}\sigma(x) + b  , \]
where $W^{(L+1)}$ is the $d_L\times d_{L+1}$ matrix (recall that
$d_{L+1}=p$) with components
\[
W^{(L+1)}_{nk} = (w^{(L+1)}_{m_k^{(L)}})_{n-1}, \qquad n=1,\ldots,p, \qquad k=1,\ldots,d_L  .
\]
Since $f_L$ is defined on $\R^{d_L}$, regarded as a finite-dimensional
subspace of $\ell_2(\N)$, denoting by $P:\ell_2(\N)\to\R^{d_L}$ the corresponding
projection 
\[
P  x = (x_{m^{(L)}_1},\ldots, x_{m^{(L)}_{d_L}})  ,
\]
for all $x\in\ell_2(\N)$ we have
\[
f_{L}(f_{L-1}(x))= f_{L}(f_{\mu_{L-1}}(x)) = f_{L}(f_{P\mu_{L-1}}(x))
 .
\]
Then,  without loss of generality, we can assume that the measure
$\mu_{L-1}$ is in  $\mathcal M(\N,\R^{d_L})$ and it has a finite
support. By iterating this procedure, we can assume that for all $\ell$
\[
\mu_\ell \in \mathcal M(\{0,\ldots,d_\ell\},\R^{d_{\ell+1}}),
\]
for some $d_0, d_1 , \dots , d_L,d_{L+1}\in \N$ (with $d_0=d$
and $d_{L+1}=p$).  This means that
\[
\mu_\ell = b^{(\ell+1)} \delta_0 + \sum_{k=1}^{d_\ell} w^{(\ell+1)}_k \delta_k  .
\]
For all $\ell=1,\dots,L+1$,  let $W^{(\ell)}$ be the $n_{d_{\ell-1}}\times
n_{d_\ell}$ matrix 
\[
W^{(\ell)}_{nk} = (w^{(\ell)}_{k})_n\qquad n=1,\ldots,d_{\ell},
\ k=1,\ldots,d_{\ell-1}  .
\]
Then  $f_\ell=f_{\mu_\ell}:\R^{d_{\ell-1}}\to \R^{d_\ell}$  is given by
\[
f_\ell(x)=
\begin{cases}
W^{(1)} x+ b^{(1)} & \ell=0 \\
  W^{(\ell+1)} \sigma(x)+ b^{(\ell+1)} & \ell>1 \quad
\end{cases},
\]
\[ 
f^{\rm deep}= f_L\circ \ldots\circ f_0,
\]
is a neural deep function according to \Cref{def:nn},
and we can rewrite $f^{\rm deep}$ as
$$
\begin{cases}
 x^{(1)} = W^{(1)} x + b^{(1)} & \in \R^{d_1} \\[5pt]
 x^{(\ell+1)} = W^{(\ell+1)} \sigma(x^{(\ell)}) + b^{(\ell+1)} & \in \R^{d_{\ell+1}} \qquad \ell = 1 , \dots , L \ ,
\end{cases}
$$
that is, $f^{\rm deep}$ is a neural network of depth $L$ and (finite) widths $ d_1 , \dots , d_L $.

\begin{rmk}\label{rmk:ass_discrete_neural_rkbs}
    For discrete neural RKBS,
\Cref{asp:x(th)} is satisfied replacing the basis functions in Definition~\ref{dfn:discrete-neural-rkbs} by $\rho_\ell(x,n) = \si(x_{n-1})\be_{n-1}$ for all $ \ell =1 , \dots, L$, with $\sigma\colon\R\to\R$ Lipschitz and $\be:\N\to\R$ a positive sequence converging to zero. Furthermore, \Cref{asp:lip} is satisfied with $
C_\ell = C_\si $ and $ g_\ell(n) = \de_{n-1} $ for all $ \ell = 1 , \dots , L $. Indeed, for all $ x , x' \in\ell^2(\N)$ and $ n\in\N$,
 \begin{equation*}
  | \rho_\ell(x,n) - \rho_\ell(x',n) | \le |\sigma(x_{n-1})-\sigma(x'_{n-1})||\be_{n-1}|\le C_\si | x_{n-1}-x'_{n-1}| |\be_{n-1}| .
 \end{equation*}
\end{rmk}

\begin{cor}[Representer theorem for discrete neural RKBS]
Under the assumptions of \Cref{thm:representer},
if $ \hh^{\rm deep} $ is a discrete neural RKBS,
then the claim of \Cref{cor:representer} holds true.
Moreover,
\begin{equation*}
 \Phi(f^{\rm deep}) \le \sum_{\ell=0}^L \sum_{k=1}^{d_\ell} \biggl( \sum_{j=1}^{d_{\ell+1}} | W^{(\ell+1)}_{jk} \be_k^{-1} |^2 \biggr)^{1/2} .
\end{equation*}
\end{cor}
\begin{proof}
Since $\Theta_0=\{0,\ldots,d\}$, without loss of generality we can
assume $d_0=d+1$ and $\theta_n^{(0)}=n$ for all
$k=0,\ldots,d_0$. Taking into account that $\rho_0(x,0)=1$ and
$\rho_0(x,k)=x_k$ if $1\leq k\le d$, we have
\[
x^{(0)} =  (1,x)\in \R\times \R^d.
\]
Since, for $\ell=1,\ldots,L$, $\Theta_\ell=\N$ and
$K_\ell(\cdot,n)=e_{n-1}$ where $\{e_n\}_{n\in\N}$ is the
canonical base of $\xx_\ell=\ell_2(\N)$, up to a permutation,
$V^{(\ell)}$ is the identity and the claim follows.
\end{proof}

The above result  shows that  deep neural networks with finite width at each hidden layer
are optimal, in the sense that 
they are solutions of  empirical risk minimization over neural RKBS.
Moreover, it provides an upper bound on the network width depending on sample size and input/output dimensions.
Finally, it shows that the regularization norm is controlled by the $\ell^1$ norm of the $\ell^2$ norms of the weights of the network.

\section{Conclusions and future work}
Studying function spaces defined by neural networks provides a natural way to understand their properties. Recently, reproducing kernel Banach spaces have emerged has a useful concept to study shallow networks.

In this paper,  we take a step towards more complex architectures considering deep networks. We allow for a wide class of activation functions and remove unnecessary low rank constraints. Our main contributions are defining classes of neural RKBS obtained composing vector-valued RKBS and deriving corresponding representer theorems borrowing ideas from \cite{doi:10.1137/21M1418642}. 

Future developments include considering more structured architectures, for example  convolutional networks, as well as investigating the statistical and computational properties of neural RKBS. Moreover, finer characterizations of the Banach structure could be obtained using the specific form of the activation function and the functional properties that this induces (see \cite{doi:10.1137/21M1418642} for the ReLU).

\section*{Acknowledgments}
L. R. acknowledges the financial support of the European Research Council (grant SLING 819789), the European Commission (Horizon Europe grant ELIAS 101120237), the US Air Force Office of Scientific Research (FA8655-22-1-7034), the Ministry of Education, University and Research (FARE grant ML4IP R205T7J2KP; grant BAC FAIR PE00000013 funded by the EU - NGEU) and the Center for Brains, Minds and Machines (CBMM), funded by NSF STC award CCF-1231216. This work represents only the view of the authors. The European Commission and the other organizations are not responsible for any use that may be made of the information it contains. The research by E.D.V. and L. R. has been supported by the MIUR grant PRIN 202244A7YL. The research by E.D.V. has been supported by the MIUR Excellence Department Project awarded to Dipartimento di Matematica, Universit\`a di Genova, CUP D33C23001110001. E.D.V. is a member of the Gruppo Nazionale per l’Analisi Matematica, la Probabilit\`a e le loro Applicazioni (GNAMPA) of the Istituto Nazionale di Alta Matematica (INdAM). This work represents only the view of the authors. The European Commission and the other organizations are not responsible for any use that may be made of the information it contains.
This work was partially supported by the MUR Excellence Department Project MatMod@TOV awarded to the Department of Mathematics, University of Rome Tor Vergata, CUP E83C18000100006. S.V. also acknowledges financial support from the MUR 2022 PRIN project GRAFIA, project code 202284Z9E4.

\printbibliography

\appendix

\section{Sparse solutions to finite-dimensional variational problems} \label{sec:appendix}

The key ingredient to establish our representer theorem
is given by a powerful variational result proved in \cite{MR4040623}.
This result deals with general minimization problems with finite-dimensional constraints
and seminorm penalization.
It states that such problems admit sparse solutions,
namely finite linear combinations of extremal points of the seminorm unit ball.
We report the formal statement below, after recalling the definition of extremal point.

\begin{dfn}[Extremal point] \label{dfn:extr-points-theor}
Let $Q$ be a convex subset of a locally convex space.
A point $ q \in Q $ is called \emph{extremal}
if $ Q \setminus \{q\} $ is convex,
that is,
there do not exist $ p,r \in Q $, $ p \ne r $,
such that $ q = tp + (1-t)r $ for some $ t \in (0,1) $.
We denote the set of extremal points of $Q$ by $ \Ext ( Q )$.
\end{dfn}

\begin{thm}[{\cite[Theorem 3.3]{MR4040623}}] \label{thm:bredies-carioni}
Consider the problem
\begin{equation} \label{eq:problem0}
 \argmin_{u \in U} F( \aa u ) + G(u) ,
\end{equation}
where
$U$ is a locally convex topological vector space,
$ \aa : U \to H  $ is a continuous, surjective linear map with values in a finite-dimensional Hilbert space $H$,
$ F : H \to (-\infty,+\infty] $
is proper, convex, coercive and lower semi-continuous,
and $ G : U \to [0,+\infty) $ is a coercive and lower semi-continuous norm.
Let $ B_U(1) $ denote the unit ball $ \{ u \in U : G(u) \le 1 \} $.
Then \eqref{eq:problem0} has solutions of the form
$$
 \sum_{k=1}^K c_k u_k,
$$
with $ u_k \in \Ext( B_U(1) ) $,
$ c_k > 0 $,
$ K \le \dim H $,
and $ \sum_{k=1}^K c_k = G(u) $.
\end{thm}

\Cref{thm:bredies-carioni} is a simplified version of \cite[Theorem 3.3]{MR4040623},
where $G$ is only assumed to be a seminorm.

In view of \Cref{thm:bredies-carioni},
in order to determine the form of sparse solutions
we need to characterize the set of extremal point of the unit ball in $U$.
In our construction of integral RKBS, $U$ is the space of vector measures
with total variation norm.
The following result provides the desired characterization for our case.
It appears in {\cite{werner1984extreme}} assuming that $\Theta$ is compact,
but the proof works analogously when $\Theta$ is locally compact.
We report the proof for the reader's convenience.

\begin{lem}[{\cite[Theorem 2]{werner1984extreme}}] \label{lem:Ext(B)}
Let $\Th$ be a locally compact, second countable topological space,
and let $\yy$ be a Banach space.
Then
$$ 
\Ext ( B_{\mm(\Th,\yy)}(1) ) = \{ y \cdot \de_{\th} : y \in \Ext(B_{\yy}(1)), \th \in \Th \} .
$$
\end{lem}
\begin{proof}
Let us denote $ E = \{ y \cdot \de_{\th} : y \in \Ext(B_{\yy}(1)), \th \in \Th \} $.
We start showing that $ \Ext ( B_\mm(1) )\subseteq E $.
Suppose that $\mu\in \Ext ( B_\mm(1) )$ but $\mu\ne y \cdot \de_{\th}$ for any $ y \in \Ext(B_\yy(1)) $ and $\th \in \Th$.
Then $|\mu|\ne \de_{\th}$ for any $\th \in \Th$,
and there is $A\in\mathcal{B}(\Theta)$ such that $0<|\mu|(A)<1$.
Denote by $\chi_A$ the indicator function on $A$.
Then, setting $t=|\mu|(A)$, $ \mu_1 = \mu\chi_A/t $ and $ \mu_2 = \mu\chi_{\Th\setminus A}/(1-t) $,  we can write $\mu$ as a convex combination
\begin{equation} \label{eq:convex_combination}
\mu = t \mu_1+(1-t) \mu_2  .
\end{equation}
Since $ t\in(0,1) $ and $\mu_1,\mu_2\in B_\mm(1)$,
we get that $ \mu \notin \Ext ( B_\mm(1) ) $, leading to a contradiction.
We now show the converse inclusion $E\subseteq\Ext ( B_\mm(1) )$.
Let $\mu=y \cdot \de_{\th}$ for some $y\in \Ext(B_{\yy}(1))$ and $ \th \in \Th $.
Suppose there are $t\in(0,1)$ and $\mu_1,\mu_2\in B_\mm(1)$ such that \eqref{eq:convex_combination}.
We want to show that necessarily $\mu_1=\mu_2=\mu$.
Consider the subspace 
\[
\zz=\{ z \cdot \de_{\th} : z \in\yy\} ,
\]
and let $\mathcal{P}\colon\mm(\Th,\yy)\to\zz$ be the projection onto $\zz$ defined by
\[
\mathcal{P} \nu=\nu(\{\theta\})\cdot \de_{\th} .
\]
By definition of total variation, for every $\nu\in\mm(\Th,\yy)$ we have
$$
\|\nu\|_{\TV} \ge \|\mathcal{P} \nu\|_{\TV}+\|\nu-\mathcal{P} \nu\|_{\TV} ,
$$
while the converse bound is simply true by triangle inequality, hence
\begin{equation}\label{eq:normssum}
\|\nu\|_{\TV} = \|\mathcal{P} \nu\|_{\TV}+\|\nu-\mathcal{P} \nu\|_{\TV} .
\end{equation}
Note that $ \mu \in \zz $ and thus $ \mathcal{P} \mu = \mu $.
Hence, applying $\mathcal{P}$ to \eqref{eq:convex_combination} we obtain
\begin{equation} \label{eq:convex_combination2}
\mu=t \mathcal{P} \mu_1+(1-t) \mathcal{P} \mu_2  .
\end{equation}
Now, consider the unit ball in $\zz$,
\[
B_\zz(1)= \{ z \cdot \de_{\th} : z \in B_\yy(1)\}  .
\]
Then $ \mu \in \Ext(B_\zz(1)) $, and $ \mathcal{P}\mu_i \in B_\zz $ for $ i = 1 , 2 $.
Therefore, looking back to \eqref{eq:convex_combination2} we must have $\mathcal{P} \mu_i =\mu$,
and in particular $\|\mathcal{P} \mu_i\|_{\TV}=\|\mu\|_{\TV}=1$.
Moreover, $ \| \mu_i \|_{\TV} \le 1 $ since $\mu_i\in B_\mm(1)$.
Thus, applying \eqref{eq:normssum} to $ \nu = \mu_i $ we get $ \|\mu_i-\mathcal{P} \mu_i\|_{\TV} = 0 $,
hence $ \mu_i = \mathcal{P} \mu_i $,
which in turn implies $ \mu_i = \mu $,
concluding the proof.
\end{proof}

The following lemma is 
    contained in the  proof of \cite[Theorem 3.2]{doi:10.1137/21M1418642}.
It allows to reduce a minimization problem over compositional functions
to a sequence of interpolation problems with respect to a generic minimizer.
Since we found it of independent interest, we thought to emphasize it
in a separate lemma. 

Let $L$ be a positive integer. Take a set $\xx_0$ and Banach
  spaces $ \xx_1 , \dots , \xx_{L+1} $.   For each $\ell=0,\ldots,L$,
  fix a Banach space $ \hh_\ell$  of  functions from $ \xx_\ell$ to $\xx_{\ell+1} $. 
To every $ f \ern{=f_0 \oplus \cdots \oplus f_L}\in  \bigoplus_{\ell=0}^L
\hh_\ell $, recall that  $ f^{\rm deep} : \xx_0 \to \xx_{L+1} $  is
defined as 
$$
 f^{\rm deep} = f_L \circ \cdots \circ f_0  .
$$
\begin{lem} \label{lem:min-interpol}
\ern{With the above setting, fix a family $x_1,\ldots, x_{N}\in\xx_0$
  and set
 $$
  \aa : \bigoplus_{\ell=0}^L \hh_\ell \to \R^N,\qquad \aa (f)_i =
  f^{\rm deep}(x_i)   .
 $$}
 For $ F : \R^N \to (-\infty , + \infty] $,
 consider the minimization problem
 \begin{equation} \label{eq:lemmin}
  \ern{\inf_{\substack{f \in \bigoplus_{\ell=0}^L \hh_\ell }} }F(\aa(f)) + \sum_{\ell=0}^L \| f_\ell \|_{\hh_\ell} .
 \end{equation}
 Assume that \eqref{eq:lemmin} has a solution \ern{$f^*=\oplus_{\ell=0}^L f_\ell^*$},
 and denote $ x_i^{(0)} = x_i $ and $ x_i^{(\ell+1)} = f^*_\ell (x_i^{(\ell)}) $ for $ \ell = 0 , \dots , L $.
 Then \ern{there exists a minimizer $\widetilde{f}= \oplus_{\ell=0}^L \widetilde{f}_\ell$
of \eqref{eq:lemmin}  such that  for all $ \ell = 0 ,\dots, L $ the function
$\widetilde{f}_\ell$ is the solution of}
 \begin{equation} \label{eq:lemint}
 \ern{ \inf_{f_\ell \in \hh_\ell} }\| f_\ell \|_{\hh_\ell} \quad \text{subject to} \quad f_\ell(x^{(\ell)}_i) = x^{(\ell+1)}_i, \qquad i = 1 \dots , N,
 \end{equation}
\ern{and  $ \| \widetilde{f}_\ell \|_{\hh_\ell} = \| f^*_\ell \|_{\hh_\ell} $.} 
\end{lem}
\begin{proof}
Let $\widetilde{f}$ be a solution to \eqref{eq:lemint}.
The solution $f^*$ satisfies the constraints $ f_\ell(x^{(\ell)}_i) = x^{(\ell+1)}_i $,
hence $ \| \widetilde{f}_\ell \|_{\hh_\ell} \le \| f^*_\ell \|_{\hh_\ell} $
and $ \widetilde{f}_\ell(x_i) = f^*_\ell(x_i) $ for all $ i = 1 , \dots , N $.
This implies $ \sum_{\ell=0}^L \| \widetilde{f}_\ell \|_{\hh_\ell} \le \sum_{\ell=0}^L \| f^*_\ell \|_{\hh_\ell} $
and $ \aa(\widetilde{f}) = \aa(f^*) $, so that
$$
 F(\aa(\widetilde{f})) + \sum_{\ell=0}^L \| \widetilde{f}_\ell \|_{\hh_\ell}
 \le
 F(\aa(f^*)) + \sum_{\ell=0}^L \| f^*_\ell \|_{\hh_\ell} .
 $$
 But since $f^*$ is a minimizer, \ern{so is $\widetilde{f}$ and} $ \| \widetilde{f}_\ell \|_{\hh_\ell} = \| f^*_\ell \|_{\hh_\ell} $.
\end{proof}

  \begin{rmk}\label{rmk:freedom}
We stress that the set $\{x_i^{(\ell)}: \ell=1,\ldots,L+1,i=1,\ldots,N\}$
defining the constraints in the problem \eqref{eq:lemint}  depends on the
choice of the minimizer $f^*=\oplus_{\ell=0}^L f_\ell$ of
the problem \eqref{eq:lemmin}, so that there is some freedom to choose the
points.
  \end{rmk}

 \ern{
 The next lemmas are needed}
to establish continuity in the setting of our
representer theorem \ref{thm:representer}. 
A continuity result akin to \Cref{lem:joint_weak*_cont}
is also needed for the proof of \cite[Theorem 3.2]{doi:10.1137/21M1418642}.
We remark that, in order for the argument to go thorough,
\emph{joint} continuity is required,
while the proof in \cite[Theorem 3.2]{doi:10.1137/21M1418642} only establishes separate continuity.
The final result remains valid since joint continuity holds true nevertheless.
This can be seen as a special case of our \Cref{lem:joint_weak*_cont},
where the last steps can be simplified
in view of the fact that, in finite-dimensional spaces, weak and strong continuity coincide.

\begin{lem} \label{lem:phi(x)_weak_cont}
 Let $\xx$ be a set,
 $\yy$ a Hilbert space,
 and $\Th$ a locally compact, second countable topological space.
 Let $ \rho : \xx \times \Th \to \R $ such that $ \rho(x,\cdot) \in \cc_0(\Th) $ for all $ x \in \xx $,
 and define $ \phi(x) : \mm(\Th,\yy) \to \yy $ by
 $$
  \phi(x)\mu = \int_{\Th} \rho(x,\th) \D\mu(\th), \qquad x \in \xx, \quad \mu \in \mm(\Th,\yy)  .
 $$
 Then, for all $ x \in \xx $,
 $\phi(x)$ is continuous
 from $\mm(\Th,\yy)$ endowed with the weak$^*$ topology
 to $\yy$ endowed with the weak topology.
\end{lem}
\begin{proof}
 Note that, since $\yy$ is Hilbert,
 the Riesz representation theorem says that $ \mm(\Th,\yy) = \cc_0(\Th,\yy)' $.
 Now, for all $ x \in \xx $, $ \mu \in \mm(\Th,\yy) $ and $ y \in \yy $, we have
\begin{equation*}
 \langle \phi(x) \mu , y \rangle_{\yy} = {}_{\cc_0(\Th,\yy)}\langle \rho(x,\cdot)y , \mu \rangle_{\mm(\Th,\yy)}  .
\end{equation*}
Thus $ \langle \phi(x) \cdot , y \rangle_{\yy} $ defines an element of the predual $ \cc_0(\Th,\yy) $,
and hence it is weakly$^*$ continuous from $ \mm(\Th,\yy) $ to $\R$.
But since this is true for all $ y \in \yy $,
it is weakly$^*$ continuous from $ \mm(\Th,\yy) $ to $\yy$ endowed with the weak topology.
\end{proof}
The following known result is a direct consequence of
  Prokhorov theorem. We report the proof for completeness.
\begin{lem}[Joint dominated convergence theorem] \label{lem:joint_dominated}
Let $\Th$ be a Polish space. \ern{
For all $ n \in \N $, let $ \la_n \in \mm(\Th) $ and $ f_n : \Th
\to \R $ continuous functions satisfying the following conditions:
\begin{enumerate}[label=\textnormal{(\roman*)},itemsep=5pt]
\item  for each $n\in\N$, the function $f_n$ is $\la_n$ integrable; 
\item the sequence $(f_n)_n$ converges to some $f:\Theta\to\R$
  uniformly on all compact sets;
\item the sequence $(f_n)_n$ is uniformly bounded; 
\item the sequence $(\la_n)_n$ converges to some $\la \in \mm(\Th)$ with
  respect to the narrow topology, {\em i.e.} 
\[ 
\lim_{n\to+\infty}   \int_\Th \varphi (\th) \D\la_n(\th)= \int_\Th
\varphi (\th) \D\la(\th),\qquad \forall\varphi \in\cc_b(\Theta)  ;
\]
\item the function $f$ is $\la$-integrable.
\end{enumerate}
Then}
$$
 \int_\Th f_n(\th) \D\la_n(\th) \xrightarrow[n\to\infty]{} \int_\Th f(\th) \D\la(\th)  .
$$
\end{lem}
\begin{proof}
For any compact $ K \subset \Th $, we have
\begin{align*}
 &
 \left| \int_\Th f_n(\th) \D\la_n(\th) - \int_\Th f(\th) \D\la(\th) \right| \\
 \leq
 &
 \left| \int_K ( f_n(\th) - f(\th) ) \D\la_n(\th) \right|
 +
 \left| \int_{\Th\setminus K} ( f_n(\th) - f(\th) ) \D\la_n(\th) \right|
 +
 \left| \int_\Th ( f_n(\th) - f(\th) ) \D\la(\th) \right| \\
 +& \left| \int_\Th f_n(\th) \D\la(\th) - \int_\Th f(\th) \D\la_n(\th) \right|\\
 \le
 &
  \sup_K | f_n - f | | \la_n|(K) 
 +
 \sup_{\Th\setminus K} | f_n - f | | \la_n|(\Th\setminus K) 
 +
  \left| \int_\Th ( f_n(\th) - f(\th) ) \D\la(\th) \right| \\
  +& \left| \int_\Th f_n(\th) \D\la(\th) - \int_\Th f(\th) \D\la_n(\th) \right| .
\end{align*}
 The first term goes to zero because $ f_n \to f $ uniformly on $K$
 and, since $(\la_n)_n$ is convergent, $ | \la_n|(K)  \le | \la_n|(\Th)  \le \sup_n \| \la_n \|_{\TV} < \infty $.
  The third term goes to zero by dominated convergence
 since $(f_n-f)_n$ is uniformly bounded.
The last term goes to zero since both terms converge to the integral $\int_\Th f(\th) \D\la(\th) $ as $n\to\infty$: the first term by the dominated convergence theorem since $(f_n)_n$ is uniformly bounded, and the second term by hypothesis $(iv)$ since $f\in C_b(\Theta)$ as uniform limit on compact sets of a uniformly bounded sequence of continuous functions.
 For the second term, fix an arbitrary $ \varepsilon > 0 $.
 Again, $ (f_n - f)_n $ is uniformly bounded.
 Moreover, since $ (\la_n)_n $ is convergent and $\Th$ is Polish,
 by the Prokhorov theorem \cite[Theorem 8.6.2]{MR2267655}
 there is a compact set $ K_\varepsilon \subset \Th $
such that $ | \la_n|(\Th\setminus K_\eps)  < \varepsilon $ for all $n$.
\ern{By taking $ K = K_\varepsilon $ we get 
\[
\limsup_{n\to\infty} \left| \int_\Th f_n(\th)
  \D\la_n(\th) - \int_\Th f(\th) \D\la(\th) \right| \leq 
\left(\sup_n\sup_{\Th } | f_n(\th) - f(\th) | \right) \eps .
\]
 The claim follows  because $\eps$ is arbitrary.}
\end{proof}

\begin{lem} \label{lem:joint_weak*_cont}
 Let $\xx_0$ be a set,
 $\xx_1, \xx_2$ separable Hilbert spaces,
 and $\Th_0, \Th_1$ locally compact, second countable topological spaces.
 For $\ell=0,1$,
 let $ \rho_\ell : \xx_\ell \times \Th_\ell \to \R $ be such that $ \rho_\ell(x,\cdot) \in \cc_0(\Th_\ell) $ for all $ x \in \xx_\ell $,
 and define $ \phi_\ell(x) : \mm(\Th_\ell,\xx_{\ell+1}) \to \xx_{\ell+1} $ by
 $$
  \phi_\ell(x)\mu = \int_{\Th_\ell} \rho_\ell(x,\th) \D\mu(\th), \qquad x \in \xx_\ell, \qquad \mu \in \mm(\Th_\ell,\xx_{\ell+1})  .
 $$
 Assume there are $ C > 0 $, $ g \in \cc_b(\Th_1,\xx_1) $ and $ \beta \in \cc_0(\Th_1) $ such that,
 for all $ x , x' \in \xx_1 $ and $ \th \in \Th_1 $,
 \begin{equation} \label{eq:Alip}
  | \rho_1(x,\th) - \rho_1(x',\th) | \le C | \langle x - x' , g(\th) \rangle_{\xx_1} | |\be(\th)|  .
 \end{equation}
 Let $ r_0,r_1 > 0 $. Then, for all $ x \in \xx_0 $, the map
 $$
  \Ga_x( \mu , \nu ) = \phi_1((\phi_0(x)\mu))\nu
 $$
 is jointly weakly$^*$ continuous from $ B_{\mm(\Th_0,\xx_{1})}(r_0) \times B_{\mm(\Th_1,\xx_{2})}(r_1) $ to $\xx_2$ endowed with the weak topology.
\end{lem}
\begin{proof}
By the Banach--Alaoglu theorem,
the product $ B = B_{\mm(\Th_0,\xx_{1})}(r_0) \times B_{\mm(\Th_1,\xx_{2})}(r_1) $ is compact.
Moreover, since $ \xx_{\ell+1} $ ($\ell=0,1$) is separable,
so is $ \cc_0(\Th_\ell,\xx_{\ell+1}) $ by the Stone--Weierstrass theorem.
Also, $ \xx_{\ell+1} $ is Hilbert, hence $ \mm(\Th_\ell,\xx_{\ell+1}) = \cc_0(\Th_\ell,\xx_{\ell+1})' $
by the Riesz representation theorem.
Therefore, $B$ is metrizable \cite[Theorem 3.16]{rudin91}.
Thus, it is enough to prove the (weak$^*$-weak) \emph{sequential} continuity of $ \Ga_x $.

To this end, let $ (\mu_n , \nu_n) \to ( \mu , \nu ) $ (weakly$^*$).
We want to show that $ \Ga_x(\mu_n , \nu_n) \to \Ga_x ( \mu , \nu ) $ (weakly).
We have
\begin{align*}
 \Ga_x(\mu_n , \nu_n) - \Ga_x ( \mu , \nu )
 =
 [\phi_1 ( \phi_0(x) \mu_n ) - \phi_1 ( \phi_0(x) \mu )]\nu_n
 +
 \phi_1(\phi_0(x)\mu) (\nu_n - \nu) \ .
\end{align*}
The second term goes to zero by \Cref{lem:phi(x)_weak_cont}.
Let us call $\ii$ the first term, and let $ z_n = \phi_0(x) \mu_n $, $ z = \phi_0(x) \mu $.
For all $ y \in \xx_2 $, by assumption \eqref{eq:Alip} we have
\begin{align*}
 | \langle \ii , y \rangle_{\xx_2} |
 & \le
 \int_{\Th_1} | \rho_1(z_n,\th) - \rho_1(z,\th) | | \be(\th) | \D | [\nu_n]_y | (\th) \\
 & \le
 C \int_{\Th_1} | \langle z_n - z , g(\th) \rangle_{\xx_1} | |\be(\th)| \D | [\nu_n]_y | (\th)  ,
\end{align*}
where $ [\nu_n]_y (E) = \langle \nu_n(E) , y \rangle_{\xx_2} $ for all $ E \in \bb(\Th_1) $.
We want to apply \Cref{lem:joint_dominated}
with $ f_n(\th) = | \langle z_n - z , g(\th) \rangle_{\xx_1} | $
and $ \la_n = |\be| | [\nu_n]_y | $,
so to conclude that $ | \langle \ii , y \rangle_{\xx_2} | \to 0 $.

First, let us verify that $ (f_n) $ is uniformly bounded. We have
$$
 f_n(\th) \le \| z_n - z \|_{\xx_1} \|g(\th)\|_{\xx_1}
 \le \sup_n \| z_n - z \|_{\xx_1} \|g\|_{\infty} ,
$$
where $ \sup_n \| z_n - z \|_{\xx_1} < \infty $ because $ (z_n-z) $ is convergent,
and $ \|g\|_{\infty} < \infty $ by assumption.
Next, we show that $ (\la_n) $ converges pointwise on $ \cc_b(\Th_1) $.
Let $ \la = |\be| | [\nu]_y | $ where $ [\nu]_y (E) = \langle \nu(E) , y \rangle_{\xx_2} $ for all $ E \in \bb(\Th_1) $.
Then, for every $ h \in \cc_b(\Th_1) $
we have $ h |\beta| \in \cc_0(\Th_1) $,
and hence,
since $ [\nu_n]_y \to [\nu]_y $ pointwise on $ \cc_0(\Th_1) $,
$$
\int_{\Th_1} h(\th) \D\la_n(\th)
= 
\int_{\Th_1} h(\th) |\be(\th)| \D[\nu_n]_y(\th)
\to
\int_{\Th_1} h(\th) |\be(\th)| \D[\nu]_y(\th)
=
\int_{\Th_1} h(\th) \D\la(\th) .
$$
Finally, we show that $ f_n \to 0 $ uniformly on compact sets.
Let $ K \subset \Th_1 $ be compact,
and fix an arbitrary $ \varepsilon > 0 $.
Since $g$ is continuous, $ g(K) $ is compact in $\xx_1$,
and thus it can be covered by a finite number of closed balls of radius $\varepsilon$.
Let $ w_1 , \dots , w_q \in \xx_1 $ be the centers of such balls,
and define $ P : \xx_1 \to \xx_1 $ as the projection onto $ \spn\{w_1,\dots,w_q\} $.
Then $ \sup_{w\in g(K)} \| w - Pw \|_{\xx_1} \le \varepsilon $.
Hence, for every $ \th \in K $ there is $ w \in P\xx_1 $ such that $ \| g(\th) - w \|_{\xx_1} \le \varepsilon $.
Thus, we have
\begin{align*}
 | f_n(\th) |
 & =
 | \langle z_n - z , g(\th) \rangle_{\xx_1} | \\
 & \le
 | \langle z_n - z , g(\th) - w \rangle_{\xx_1} |
 +
 | \langle P(z_n - z) , w \rangle_{\xx_1} | \\
 & \le
 \| z_n - z \|_{\xx_1} \| g(\th) - w \|_{\xx_1}
 +
 \| P(z_n - z) \|_{\xx_1} \| w \|_{\xx_1} \\
 & \le
 \| z_n - z \|_{\xx_1} \| g(\th) - w \|_{\xx_1}
 +
 \| P(z_n - z) \|_{\xx_1} \left( \| g(\th) - w \|_{\xx_1} + \| g(\th) \|_{\xx_1} \right) \\
 & \le
 \| z_n - z \|_{\xx_1} \varepsilon
 +
 \| P(z_n - z) \|_{\xx_1} \left( \varepsilon + \| g \|_\infty \right)  .
\end{align*}
Now, since $ z_n \to z $ weakly, $ \sup_n \| z_n - z \|_{\xx_1} < \infty $,
and $ \| P(z_n - z) \|_{\xx_1} \to 0 $ because $P$ has finite rank.

All the assumptions of \Cref{lem:joint_dominated} are therefore satisfied,
and its application concludes the proof.
\end{proof}

  \begin{rmk}\label{rmk:evaluation}
In the framework of deep integral RKBS, the above lemma states that the
evaluation functional at  a fixed $x\in\xx_0$ 
\[
(\mu , \nu)\mapsto f_{\nu}(f_\mu(x))
\]
is jointly weakly$^*$ continuous from $ B_{\mm(\Th_0,\xx_{1})}(r_0)
\times B_{\mm(\Th_1,\xx_{2})}(r_1) $ to $\xx_2$ endowed with the weak
topology. The proof for $L=2$ can be easily extended to cover the case
$L>2$. 
  \end{rmk}
  
\end{document}